\documentclass{article}

\usepackage[preprint]{neurips_2026}

\usepackage[utf8]{inputenc}
\usepackage[T1]{fontenc}
\usepackage{hyperref}
\usepackage{url}
\usepackage{booktabs}
\usepackage{amsfonts}
\usepackage{amsmath,amssymb,amsthm}
\usepackage{nicefrac}
\usepackage{microtype}
\usepackage{xcolor}
\usepackage{graphicx}
\usepackage{enumitem}
\usepackage{cleveref}
\usepackage{algorithm}
\usepackage{algorithmic}

\newtheorem{theorem}{Theorem}[section]
\newtheorem{proposition}[theorem]{Proposition}

\newtheorem{lemma}[theorem]{Lemma}
\newtheorem{definition}[theorem]{Definition}
\newtheorem{axiom}{Axiom}
\newtheorem{remark}[theorem]{Remark}

\newcommand{\R}{\mathbb{R}}
\newcommand{\Z}{\mathbb{Z}}

\newcommand{\Od}{\mathrm{O}(d)}

\newcommand{\SO}{\mathrm{SO}}
\newcommand{\GL}{\mathrm{GL}}
\newcommand{\diag}{\mathrm{diag}}
\newcommand{\Att}{\mathrm{Att}}
\newcommand{\Comm}{\mathrm{Comm}}

\title{Journey Operators for Structured Multi-Axis Composition}

\author{Mahesh Godavarti\\
A Carrot, Inc}

\begin{document}

\maketitle

\begin{abstract}
Many kinds of data have structure along one or more axes: words in a sentence, pixels in an image, nodes in a tree, frames in audio, or cells in a 3D volume.
Along one axis, order matters: ``the dog bit the man'' is different from ``the man bit the dog.''
Across independent axes, however, neither composition nor movement should depend on the order of axes: in an image, composing right then down should give the same result as composing down then right, and moving right then down should describe the same relative position as moving down then right.

We develop a framework for modeling this kind of multi-axis structure.
Each data item carries its content together with a small transformation for each axis.
A path connecting two positions defines a \emph{journey}; the \emph{journey operator} is the product of per-axis transformations along that path, governing both how data composes along the path and how relative position is described.
When the transformations are fixed, our framework recovers Rotary Position Embedding (RoPE) and its multi-dimensional variants.
When they depend on the data, the model gains a content-adaptive positional inductive bias.

We show exactly when these paths are well-defined: both composition and movement across axes are path-independent precisely when the axis transformations commute.
We also prove that, under the stated toral-frame symmetry, cocycle, bilinearity, and norm-preservation assumptions, the resulting pairwise scoring rule must take the form of block-wise rotations, explaining why RoPE-like methods arise naturally.

Finally, we use this theory to design JoFormer, a model for value aggregation, and relate it to attention and state-space models (SSMs).

Initial experiments across vision, language, and length generalization suggest that these inductive biases can have observable consequences in practice.
\end{abstract}

\section{Introduction}\label{sec:intro}

Any data that can be indexed---sequences, images, audio, volumetric arrays, trees---composes along multiple axes.
In the one-dimensional (1D) case, a sequence like $[a\; b\; c\; d]$ can be formed by composing sub-sequences---for example, by concatenating $[a\; b]$ with $[c\; d]$, or $[a]$ with $[b\; c\; d]$. Algebraic tools such as non-commutative semigroups or free groups have long provided principled ways to model such 1D compositional structure~\citep{Rudolph2010}.

In contrast, there exists no broadly accepted algebraic framework for modeling \emph{two-dimensional (2D) composition}. Consider the array
$\bigl[\begin{smallmatrix} a & b \\ c & d \end{smallmatrix}\bigr]$:
it can be composed vertically, by stacking $[a\; b]$ over $[c\; d]$, or horizontally, by placing $\bigl[\begin{smallmatrix} a \\ c \end{smallmatrix}\bigr]$ next to $\bigl[\begin{smallmatrix} b \\ d \end{smallmatrix}\bigr]$.
Such multiple valid composition paths do not fit neatly into existing algebraic systems, which are typically designed for linear (1D) structured data.
Even in the 1D case, aligning algebraic formalisms with the needs of modern machine learning architectures can be problematic---for example, representing tokens as matrices instead of vectors breaks the core assumptions of the attention mechanism in transformers, which relies fundamentally on vector operations.
This tension reveals a deeper issue: the absence of an algebraic framework that can both align with vector-based learning architectures and naturally support composition along multiple axes.

We present such an algebraic framework, built on one operation that models composition in 1D and extends naturally to multiple axes.
In 1D, each data point carries both content and a single orthogonal transformation---its \emph{axis-step generator}.
Composing two data points applies the first point's axis-step generator to rotate the second's content, and multiplies their axis-step generators together.
This is associative but non-commutative: the first element's axis-step generator acts on the second's content, not vice versa, so order matters.
Composing an entire sequence accumulates these transformations: each value is transformed by the product of all preceding axis-step generators (\Cref{ssec:composition}).

In nD, each datum carries one axis-step generator per axis.
Composing two data points along a given axis applies that axis's axis-step generator to rotate the second point's content into the first's frame, accumulates the axis-step generator for that axis, and leaves all other axes unchanged.
For multi-axis data like images, we want two kinds of path independence: composition itself should not depend on the order of axes (composing right-then-down equals down-then-right), and any movement should depend only on the initial and final positions, not on the path taken.
This is the inductive bias---it encodes the assumption that independent axes of real-world data commute.
Our framework captures both: when per-axis axis-step generators commute (the \emph{flat} regime), composition is path-independent and the \emph{journey}---the relative transformation between two positions---depends only on displacement (\Cref{thm:interchange}).
In 1D with a single axis-step generator, this recovers RoPE; content-dependent axis-step generators yield data-dependent composition (JoFormer-projected).

This compositional structure is absent from standard attention, which aggregates values without non-commutative composition.
Bringing composition into attention---applying the journey operator to values before aggregation---is the architecture (JoFormer) developed in \Cref{sec:joformer}.

Initial experiments across vision, language, and length generalization are consistent with the framework's value (\Cref{sec:experiments}).
With fixed rotations, the framework is a generalized DFT with learned frequencies.
RoPE~\citep{su2021roformer}, 2D-RoPE~\citep{heo2024rope2d}, tree PEs~\citep{shiv2019novel}, and RotatE~\citep{sun2019rotate} emerge as special cases; the value-path formula provides common notation for standard attention, SSMs, and JoFormer (\Cref{app:value-path}).

\paragraph{Contributions.}
\begin{enumerate}[nosep]
\item We introduce a \emph{compositional framework} (\Cref{ssec:composition}): in 1D, each data point carries content and a single axis-step generator (an orthogonal transformation); in nD, each data point carries one axis-step generator per axis. Composition along any axis is a single associative operation.
\item From this framework we derive several theoretical results:
  \begin{enumerate}[nosep,label=(\alph*)]
  \item path-independent composition requires commuting axis-step generators (\Cref{thm:interchange});
  \item under natural axioms (bilinearity, toral frame invariance, cocycle, norm preservation), the only compatible mechanism is block-diagonal $\SO(2)^{d/2}$ rotation (\Cref{thm:toral}); if the full orthogonal group $\Od$ is the symmetry instead, positional structure collapses (\Cref{thm:od-impossibility});
  \item in the flat (commutative) regime, value transport via the journey operator is translation equivariant (\Cref{thm:v-rotation});
  \item the DFT is a special case of the compositional embedding, and learned frequencies yield a generalized Fourier transform (\Cref{prop:dft}).
  \end{enumerate}
\item These constraints lead to \emph{JoFormer} (\Cref{sec:joformer}), an architecture that brings composition into attention via data-dependent projected angles and value-side rotation.
\item A \emph{value-path hierarchy} (\Cref{app:value-path}) places standard attention, SSMs, and JoFormer on a common spectrum---the SSM recurrence unrolls into the same value-path formula (\Cref{prop:ssm-bridge})---clarifying what each family composes and what it discards.
\item Experiments on vision, language modeling, and length generalization suggest the framework has practical value (\Cref{sec:experiments}).
\end{enumerate}

\section{Algebraic Framework and Theoretical Analysis}\label{sec:framework}

\subsection{Compositional Data Points}\label{ssec:composition}

Everything in this paper flows from one operation.
In 1D, each data point is a pair $(\mathbf{v}, R)$: content $\mathbf{v} \in \R^d$ plus an orthogonal transformation $R \in G \leq \mathrm{O}(d)$ (the group of norm-preserving $d \times d$ matrices)---its \emph{axis-step generator}.
Given two such points $(\mathbf{a}, A)$ and $(\mathbf{b}, B)$---where $\mathbf{a}, \mathbf{b} \in \R^d$ are contents and $A, B \in G$ are axis-step generators---their composition is:
\begin{equation}\label{eq:composition}
(\mathbf{a}, A) \circ (\mathbf{b}, B) \;:=\; (\mathbf{a} + A\mathbf{b},\; AB).
\end{equation}
The content of the composite is $\mathbf{a}$ plus $\mathbf{b}$ rotated into $\mathbf{a}$'s frame by $A$; the axis-step generator of the composite is the product $AB$.
This operation is associative with identity $(\mathbf{0}, I)$, but non-commutative: the first element's axis-step generator acts on the second's content, so order matters.

\paragraph{Axis-specific composition.}
For data indexed by $D$ axes, each datum carries content and one axis-step generator per axis: $(\mathbf{v}, R_1, \ldots, R_D)$.
Given two data points with contents $\mathbf{a}, \mathbf{b} \in \R^d$ at axis-$k$ positions $n_k$ and $m_k$ respectively, composing along axis $k$ applies the first point's axis-$k$ generator to rotate the second's content and accumulates that axis-step generator, leaving other axes unchanged:
\begin{align*}
&(\mathbf{a}, R_1^{n_1}, \ldots, R_k^{n_k}, \ldots, R_D^{n_D}) \circ_k (\mathbf{b}, R_1^{n_1}, \ldots, R_k^{m_k}, \ldots, R_D^{n_D}) \\
&\qquad= (\mathbf{a} + R_k^{n_k}\mathbf{b},\; R_1^{n_1}, \ldots, R_k^{n_k + m_k}, \ldots, R_D^{n_D}).
\end{align*}
Crucially, $\circ_k$ is defined only when the two operands agree on all axes $i \neq k$---they must share the same exponents $(n_1, \ldots, n_{k-1}, n_{k+1}, \ldots, n_D)$.
Composition therefore operates over \emph{tiles}: 1D slices of the grid at fixed coordinates on the remaining $D{-}1$ axes.
One cannot, for instance, compose elements from different rows of a 2D grid along the column axis.
When $R_k R_l = R_l R_k$ for all axis pairs, the regime is \emph{flat}: composition is path-independent (composing along axis $k$ then $l$ gives the same result as $l$ then $k$), and the relative transformation between two positions depends only on their displacement, not on which path connects them (\Cref{thm:interchange}).
Both properties follow from the single condition of commuting axis-step generators.
When axis-step generators do not commute, the regime is \emph{curved}: composition depends on the order of axes, and relative transformations depend on absolute position.

\paragraph{Sequence unrolling.}
In 1D, composing a sequence $[e_1, \ldots, e_T]$ with $e_t = (\mathbf{v}_t, R_t)$ gives:
\begin{equation}\label{eq:unrolling}
e_1 \circ e_2 \circ \cdots \circ e_T \;=\; \Bigl(\,\sum_{i=1}^{T} \Bigl(\prod_{j=1}^{i-1} R_j\Bigr) \mathbf{v}_i,\;\; \prod_{j=1}^{T} R_j\Bigr).
\end{equation}
The content part $\sum_i (\prod_{j<i} R_j)\mathbf{v}_i$ is a weighted sum where each value is rotated by all preceding transformations---position is encoded through accumulated rotation.

\paragraph{The journey operator.}
The \emph{journey operator} from position $j$ to position $k$ is defined as the composition of axis-step generators along the path connecting them.
In 1D, traveling from $j$ to $k$ (with $j > k$) means composing the generators at positions $k, k{+}1, \ldots, j{-}1$ in sequence.

\paragraph{Absolute operators.}
Separately, define the \emph{absolute operator} at position $i$ as $A_i = \prod_{j=1}^{i-1} R_j$---so $A_1 = I$, $A_2 = R_1$, $A_3 = R_1 R_2$, and so on.
The unrolled content~\eqref{eq:unrolling} can then be rewritten as $\sum_i A_i \mathbf{v}_i$: each value $\mathbf{v}_i$ is expressed in the global frame (position~1's frame) via $A_i$.

\paragraph{Journey in terms of absolute operators.}
We can show that the path-defined journey equals $A_k^{-1} A_j$.
To see this: the absolute operator $A_j$ maps $\mathbf{v}_j$ into the global frame, and $A_k^{-1}$ maps from the global frame into position $k$'s frame.
Their composition $P_{j \to k} = A_k^{-1} A_j$ therefore transforms $\mathbf{v}_j$ into $k$'s frame---exactly what the path-defined journey does.
When all axis-step generators are identical ($R_t = R$), the absolute operator is $A_i = R^{i-1}$ and the journey reduces to $P_{j \to k} = R^{-(k-1)} R^{j-1} = R^{j-k}$, a pure function of displacement.
In nD, the absolute operator at position $\mathbf{n}$ is $A_\mathbf{n} = R_1^{n_1} \cdots R_D^{n_D}$; the formal multi-axis treatment is in \Cref{ssec:path-transport}.

\paragraph{Value-path formula (1D).}
The unrolled content gives a sum of rotated values, each weighted equally.
In an attention mechanism, different positions contribute differently: position $j$'s value is weighted by an attention weight $\alpha_{kj}$ (how much position $k$ attends to position $j$).
Combining the journey operator with selective attention yields the value-path formula:
\begin{equation}\label{eq:unified}
\mathbf{c}_k = \sum_{j} \alpha_{kj}\, P_{j \to k}\, \mathbf{v}_j,
\end{equation}
where $P_{j \to k} = A_k^{-1} A_j$ is the journey from $j$ to $k$.
This notation unifies standard attention ($P = I$, no compositional structure), SSMs ($P = \prod A_{x_t}$, content-dependent recurrence with uniform weights), and JoFormer ($P = A_k^{-1} A_j$ with full attention weights).

\paragraph{2D composition and value-path formula.}
For a 2D grid of data points $e_{(i,j)} = (\mathbf{v}_{(i,j)}, R_x, R_y)$, we compose along rows first (1D composition along $x$): each row $i$ gives $\sum_j R_x^{\,j}\, \mathbf{v}_{(i,j)}$ by the 1D unrolling~\eqref{eq:unrolling}.
Composing the row results along columns (1D composition along $y$) applies $R_y^{\,i}$ to each row's content:
\[
\sum_{i} R_y^{\,i} \sum_{j} R_x^{\,j}\, \mathbf{v}_{(i,j)}
\;=\; \sum_{i,j} R_y^{\,i}\, R_x^{\,j}\, \mathbf{v}_{(i,j)}
\;=\; \sum_{i,j} A_{(i,j)}\, \mathbf{v}_{(i,j)},
\]
where $A_{(i,j)} = R_y^{\,i}\, R_x^{\,j}$ is the absolute operator at position $(i,j)$.
When the axis-step generators commute, composing columns-then-rows gives the same result (since $R_y^{\,i} R_x^{\,j} = R_x^{\,j} R_y^{\,i}$).
This generalizes directly to $D$ axes: composing along each axis in turn yields $\sum_{\mathbf{n}} A_{\mathbf{n}}\, \mathbf{v}_{\mathbf{n}}$ with $A_{\mathbf{n}} = \prod_{k=1}^D R_k^{\,n_k}$, independent of the order of axes when generators commute.
The journey from $(i', j')$ to $(i, j)$ is:
\[
P_{(i',j') \to (i,j)} = A_{(i,j)}^{-1}\, A_{(i',j')} = R_x^{-j}\, R_y^{-i}\, R_y^{\,i'}\, R_x^{\,j'}.
\]
When the axis-step generators commute ($R_x R_y = R_y R_x$), this simplifies to $R_y^{\,i'-i}\, R_x^{\,j'-j}$---a pure function of displacement.
The value-path formula becomes:
\begin{equation}\label{eq:unified-2d}
\mathbf{c}_{(i,j)} = \sum_{(i',j')} \alpha_{(i,j),(i',j')}\; R_y^{\,i'-i}\, R_x^{\,j'-j}\; \mathbf{v}_{(i',j')}.
\end{equation}
This generalizes directly to $D$ axes: the absolute operator at position $\mathbf{n} = (n_1, \ldots, n_D)$ is $A_{\mathbf{n}} = \prod_{k=1}^D R_k^{\,n_k}$, and the journey (with commuting axis-step generators) is the multi-axis displacement operator
$P_{\mathbf{n}' \to \mathbf{n}} = \prod_{k=1}^D R_k^{\,n'_k - n_k}$.

The hierarchy is developed in \Cref{app:value-path}.

\subsection{Path Transport and Commutativity}\label{ssec:path-transport}

We now formalize how composing axis-step generators along different paths may or may not yield the same result.

\begin{definition}[Multi-axis composition and path transport]\label{def:position-model}
For $D$ axes with axis-step generators $R_1, \ldots, R_D \in G \leq \mathrm{O}(d)$, a position $\mathbf{n} = (n_1, \ldots, n_D) \in \Z^D$ has absolute operator $A_\mathbf{n} = R_1^{n_1} \cdots R_D^{n_D}$ (as in \Cref{ssec:composition}).
For a word $\pi = k_1 \cdots k_m$ over axes $\{1, \ldots, D\}$ (a path), define the \emph{path transport}:
\begin{equation}\label{eq:path-transport}
T(\pi) = R_{k_m} \cdots R_{k_1}.
\end{equation}
\end{definition}

The path transport $T(\pi)$ composes generators step by step; different orderings of the same axis steps may yield different composites---the transport is \emph{path-dependent} in general.
Under the canonical axis ordering, the composite transport from $\mathbf{m}$ to $\mathbf{n}$ evaluates to $A_\mathbf{n}^{-1} A_\mathbf{m}$, but this equals the transport along every other path only when generators commute.
This motivates the central theorem.
\begin{theorem}[Path independence $\Leftrightarrow$ commuting generators]\label{thm:interchange}
If you compose along a multi-axis grid using per-axis axis-step generators, the result is independent of the order you take the steps (path-independent) if and only if the axis-step generators commute.
When they commute, the journey operator $P_{\mathbf{m} \to \mathbf{n}} = A_\mathbf{n}^{-1} A_\mathbf{m}$ depends only on displacement---it is fully determined by composition.
Formal statement and proof in \Cref{app:interchange-proof}.
\end{theorem}

\begin{remark}[Connection to interchange laws]
In higher category theory, the \emph{interchange law} $(x \circ_i y) \circ_j (z \circ_i w) = (x \circ_j z) \circ_i (y \circ_j w)$ states that two composition operations are compatible.
Path independence is the geometric manifestation of this algebraic condition: the interchange law holds for group-valued elements iff generators commute iff transport is path-independent.
\end{remark}

\subsection{Journey Operator and Score/Value Attention}\label{ssec:journey}

As derived in \Cref{ssec:composition}, composing axis-step generators along each axis yields absolute operators $A_i, A_j \in G \leq \mathrm{O}(d)$ at each position.
In the flat (commutative) regime established by \Cref{thm:interchange}, the journey operator arising from composition reduces to:
\begin{equation}\label{eq:journey}
P_{j \to i} = A_i^{-1} A_j.
\end{equation}
This operator structures both the score and value sides of attention:
\begin{align}
\text{Score:} \quad & \Att(i,j) = q_i^\top P_{j \to i}\, k_j = (A_i q_i)^\top (A_j k_j), \label{eq:score}\\
\text{Value:} \quad & \mathbf{c}_i = A_i^{-1} \textstyle\sum_j \alpha_{ij}\, A_j\, v_j = \sum_j \alpha_{ij}\, P_{j \to i}\, v_j. \label{eq:value-agg}
\end{align}
For orthogonal $A_i$: $(A_i q_i)^\top (A_j k_j) = q_i^\top A_i^{-1} A_j k_j = q_i^\top P_{j\to i} k_j$, confirming consistency.
Score computation rotates $q$ and $k$ by their respective absolute operators; value computation rotates each $v_j$ by $A_j$ then inverse-rotates the output by $A_i^{-1}$.

\textbf{RoPE as a special case.}
For a sequence, $A_t = R^t$ with $R = \diag(R(\theta_1), \ldots, R(\theta_{d/2}))$ being a block-diagonal rotation.
Then $P_{j \to i} = R^{-i} R^j = R^{j-i}$, and the score becomes $(R^i q_i)^\top (R^j k_j) = q_i^\top R^{j-i} k_j$---exactly RoPE~\citep{su2021roformer}.
Standard RoPE uses only the score side; our framework additionally rotates values.

\textbf{2D-RoPE as a special case.}
For a 2D grid with generators $R_x, R_y$ acting on disjoint subspaces (first $d/4$ planes for $x$, remaining for $y$):
$A_{(m,n)} = R_x^m R_y^n = \diag(R(\omega_1 m), \ldots, R(\omega_{d/4} m), R(\phi_1 n), \ldots, R(\phi_{d/4} n))$.
Since the generators act on disjoint planes, $R_x R_y = R_y R_x$---path independence is satisfied by construction.

\subsection{Toral Classification}\label{ssec:toral}

We classify all bilinear, norm-preserving, cocycle-compatible attention mechanisms under toral frame symmetry $T = \SO(2)^{d/2}$.
The result: the score matrix must be a block-diagonal $\SO(2)^{d/2}$ rotation---each of $d/2$ planes gets an independent angle $\theta_k(s) - \theta_k(s')$.
This recovers RoPE~\citep{su2021roformer}, 2D-RoPE~\citep{heo2024rope2d}, and tree PEs~\citep{shiv2019novel} as special cases.
If the full orthogonal group $\Od$ is the symmetry instead, positional structure collapses entirely.
Full axioms (bilinearity, toral frame invariance, cocycle, norm preservation), formal statements, and proofs are in \Cref{app:toral}.

\begin{theorem}[Toral Classification]\label{thm:toral}
Under bilinearity, toral frame invariance, cocycle compositionality, and norm preservation, the score matrix must be a block-diagonal rotation:
each of $d/2$ planes gets an independent angle $\theta_k(s) - \theta_k(s')$, and the full operator is their direct sum.
This is the unique form compatible with the symmetry contract.
Formal statement and proof in \Cref{app:toral-proof}.
\end{theorem}

\begin{theorem}[$\Od$ Impossibility]\label{thm:od-impossibility}
If the full orthogonal group $\Od$ is the frame symmetry (instead of the torus), every score matrix reduces to a scalar multiple of the identity---positional structure is completely lost.
Proof in \Cref{app:od}.
\end{theorem}

\subsection{Flat/Curved Dichotomy and the V Rotation Prediction}\label{ssec:flat-curved}

\begin{definition}[Flat and curved]
A multi-axis compositional framework is \emph{flat} if its per-axis axis-step generators commute ($[R_i, R_j] = 0$ for all $i \neq j$), and \emph{curved} otherwise.
\end{definition}

All block-diagonal $\SO(2)^{d/2}$ methods are flat; dense rotation matrices (like LieRE) are generically curved.
In the flat regime, $P_{j \to i} v_j = R(\Delta p) v_j$ depends only on relative displacement---the journey operator is fully determined by composition.
\begin{definition}[Translation equivariance]\label{def:translation-equiv}
For a shift $\mathbf{u} \in \Z^D$, define the translation operator $(T_\mathbf{u} v)_\mathbf{m} = v_{\mathbf{m} - \mathbf{u}}$.
An aggregation rule $F$ mapping value fields $v$ to output fields $\mathbf{c}$ is \emph{translation equivariant} if $F(T_\mathbf{u} v)_{\mathbf{n} + \mathbf{u}} = F(v)_\mathbf{n}$ for all $\mathbf{n}, \mathbf{u}$: shifting inputs shifts outputs by the same amount.
\end{definition}

\begin{theorem}[Relative-displacement equivariance of value transport]\label{thm:v-rotation}
When per-axis generators commute (flat regime) and attention weights depend only on relative displacement, applying the journey operator $A_\mathbf{n}^{-1} A_\mathbf{m}$ to values before aggregation gives a translation-equivariant rule: shifting all positions by the same amount shifts outputs identically.
When generators do \emph{not} commute, the value transport becomes path-dependent---it depends on absolute positions, not just displacement.
Formal statement and proof in \Cref{app:v-rotation-proof}.
\end{theorem}

\begin{remark}[Scope: operator-level, not whole-layer equivariance]\label{rem:operator-scope}
In a standard transformer, attention weights $\alpha_{ij} = \mathrm{softmax}_j(q_i^\top P_{j \to i} k_j / \sqrt{d})$ depend on token content through $q_i = W_Q x_i$ and $k_j = W_K x_j$, so they are not purely displacement-dependent.
\Cref{thm:v-rotation} therefore does not imply that the full attention layer is translation equivariant.
What the theorem does justify is the \emph{value-side operator choice}: once absolute operators are determined by composition (\Cref{ssec:composition}), $A_i^{-1} A_j$ is the canonical cocycle-compatible journey that yields displacement-dependent value transport when weights happen to be displacement-only.
In the general content-dependent case, no full equivariance guarantee remains, but the value transform is still norm-preserving and uses the same toral operator family.
\end{remark}

We test the flat-regime hypothesis experimentally in \Cref{ssec:cifar}.

\subsection{Value-Path Hierarchy and Relation to SSMs}\label{ssec:value-path}

The value-path formula~\eqref{eq:unified}, introduced in \Cref{ssec:composition} as a direct consequence of the composition operation, provides common notation for standard attention, SSMs, and JoFormer.
This is a notation-level bridge rather than an equivalence: Mamba's transition matrices are not generally toral, attention weights are normalized, and the forward SSM kernel differs from the attention journey by convention.
The hierarchy, SSM bridge proposition, and detailed analysis are in \Cref{app:value-path}.

\section{JoFormer Architecture}\label{sec:joformer}

The algebraic framework motivates a concrete architecture: the \emph{JoFormer} (Journey-based Transformer), which implements data-dependent journey operators on both the score and value sides.

\subsection{Design Principles}\label{ssec:design}

Three inductive biases follow from the theory:
\begin{enumerate}[nosep]
\item \textbf{Commutativity}: use block-diagonal $\SO(2)^{d/2}$ rotations (flat, satisfying path independence).
Under the full modeling contract (bilinearity, toral-frame invariance, cocycle, norm preservation), this is the only compatible structure (\Cref{thm:toral}). The choice of $T = \SO(2)^{d/2}$ as the symmetry group is an assumption of that contract, not a consequence.
\item \textbf{V rotation}: apply $P_{j \to i}$ to values (justified by the flat/curved analysis---the journey operator arising from composition gives displacement-dependent value transport when axis-step generators commute).
\item \textbf{Data dependence}: compute angles $\theta(x)$ from content, enriching the value-path beyond fixed positional functions---an SSM-like data-dependent value transform, though without the recurrent product over intermediate states.
\end{enumerate}

\paragraph{What the theory does and does not justify.}
The core claim is that standard attention lacks a structural inductive bias on the value side: commutative summation discards compositional context.
\Cref{thm:interchange,thm:v-rotation} prove that the journey operator---arising directly from composition---fills this gap, giving displacement-dependent value transport exactly in the flat/commutative regime with relative attention weights.
JoFormer-fixed and JoFormer-learned preserve this contract.
JoFormer-projected keeps the same norm-preserving toral operator family but makes angles content-dependent; it is therefore a theory-\emph{motivated} architecture, not a direct translation-equivariance corollary.

A detailed breakdown of which theorem supports which variant is in \Cref{tab:variant-scope} (\Cref{app:variant-scope}).

\subsection{Architecture Variants}\label{ssec:variants}

All variants share the same attention computation:
\begin{align}
\text{Score:} \quad & \alpha_{ij} \propto \exp\big((R(\theta_i) q_i)^\top (R(\theta_j) k_j) / \sqrt{d}\big), \label{eq:joformer-score} \\
\text{Value:} \quad & c_i = R(\theta_i)^{-1} \textstyle\sum_j \alpha_{ij}\, R(\theta_j)\, v_j. \label{eq:joformer-value}
\end{align}
The variants differ in how angles $\theta_i$ are computed:

\paragraph{JoFormer-fixed.}
Angles are linear in position: $\theta_k^{(l)}(t) = \omega_k \cdot t$ (same as RoPE frequencies).
This is equivalent to RoPE on Q/K \emph{plus} V rotation and inverse rotation on output.
Implements the journey value path $P_{j \to i} = R^{-i}R^j = R^{j-i}$.

\paragraph{JoFormer-learned.}
Per-layer learned frequency vectors $\omega^{(l)} \in \R^{d/2}$ define $\theta^{(l)}(t) = t \cdot \omega^{(l)}$.
This is equivalent to RoPE with per-layer learned frequencies instead of fixed geometric spacing.
Enables layer-specific frequency selection while maintaining the linear-in-position structure.

\paragraph{JoFormer-projected.}
Per-layer MLP angle projectors compute angles from the residual stream:
\begin{equation}\label{eq:projected}
\theta^{(l)}(x) = W_2^{(l)} \,\mathrm{GELU}(W_1^{(l)} \,\mathrm{LN}(x)),
\end{equation}
with $W_1^{(l)} \in \R^{d \times d}$, $W_2^{(l)} \in \R^{d/2 \times d}$.
Angles are \emph{content-dependent}: computed fresh at each layer from the current representation.
This provides an SSM-like data-dependent value transform: $P_{j \to i}$ depends on the current residual representations at positions $i$ and $j$, but does not implement the recurrent product over intermediate states.
The MLP adds $\sim 1.5d^2$ parameters per layer (small relative to the $12d^2$ for attention + FFN).
In practice, JoFormer-projected requires softmax attention (softplus is unstable with data-dependent angles, as unbounded weights accumulate across layers) and a lower learning rate ($\leq 2{\times}10^{-4}$).

\paragraph{Attention ordering for vision.}
JoFormer uses $K^\top Q$ (not $Q^\top K$) attention ordering in vision applications.
Since $q_i^\top P_{j\to i} k_j = (A_i q_i)^\top (A_j k_j) = \hat{q}_i^\top \hat{k}_j$, both orderings yield equivalent scores; we adopt $K^\top Q$ as a convention matching the reference-factored structure $\hat{k}_j^\top \hat{q}_i$.

\section{Experiments}\label{sec:experiments}

Each experiment is a single-seed sanity check for a different theoretical hypothesis; none are benchmark claims.
MNIST (\Cref{app:mnist-results}): monoidal compression.
CIFAR-100/ImageNet: V rotation.
Wikipedia LM: value-path hierarchy.
Length generalization: projected angles.

\subsection{CIFAR-100 V Rotation and Scaling}\label{ssec:cifar}

\paragraph{Setup.}
Two experimental setups test V rotation for vision. \textbf{(A)} ViT-Tiny ($D{=}384$, 12 layers, 6 heads, ${\sim}14.9$M params, patch $4{\times}4$, CIFAR-100, 200 epochs, Adam lr$=10^{-4}$, cosine annealing, H100, seed=42) using LieRE's framework~\citep{liere2024}.
\textbf{(B)} Smaller ViT (4 layers, 4 heads, patch $4{\times}4$, CIFAR-100, 300 epochs, cosine lr$=10^{-3}$, dropout=0.1, weight decay, mixup, cutout) at $D \in \{32, 64, 128, 256\}$ with fully deterministic GPU-resident training.
Both setups use exact reproducibility (fixed seeds, no DataLoader workers, torch-op augmentation on GPU).

\paragraph{Methods.}
Nine PE variants organized on two axes: \emph{frequency type} (fixed RoPE vs.\ learned) and \emph{V rotation} (Q/K only vs.\ Q/K/V with inverse rotation on output).
The axial factorization splits dimensions into disjoint $y$- and $x$-subspaces, guaranteeing commutativity.

\paragraph{Results: V rotation is consistently beneficial in these paired single-seed CIFAR protocols.}

\begin{table}[h]
\centering
\caption{V rotation effect on CIFAR-100 (paired comparisons, Setup B, D=256, deterministic single run). V rotation helps in all five paired comparisons; the synergy with learnable frequencies is largest.}
\label{tab:v-rotation-cifar}
\begin{tabular}{@{}lccr@{}}
\toprule
\textbf{Approach} & \textbf{Q/K only} & \textbf{Q/K/V} & \textbf{$\Delta$} \\
\midrule
Axial learned & monoidal\_axial: 60.92\% & \textbf{joformer\_axial: 63.27\%} & \textbf{+2.35\%} \\
Axial learned (per-layer) & 60.63\% & 62.76\% & +2.13\% \\
Axial fixed & rope2d: 61.39\% & joformer\_old: 61.85\% & +0.46\% \\
Combined fixed & rope2dv2: 55.53\% & joformer\_fixed: 56.66\% & +1.13\% \\
Combined learned (per-layer) & 60.85\% & 61.17\% & +0.32\% \\
\bottomrule
\end{tabular}
\end{table}

V rotation helps consistently across D=32--256 (\Cref{tab:cifar-scaling} in \Cref{app:cifar-scaling}), with V rotation and learnable frequencies synergistic ($+1.88\%$ at D=256).
In the ViT-Tiny LieRE framework, axial-dense V rotation gains $+0.51\%$, while LieRE64's dense (curved) rotation drops $1.23\%$ (\Cref{app:experiments})---consistent with the flat-regime hypothesis.

\subsection{ImageNet ViT-S Scale Validation}\label{ssec:imagenet}

ViT-S ($D{=}384$, 12 layers, 6 heads, ${\sim}22$M params, DeiT-III recipe~\citep{touvron2022deit3}, 300 epochs, single seed; details in \Cref{app:experiments}).

\begin{table}[h]
\centering
\caption{ImageNet-1K ViT-S (DeiT-III, 300 epochs, single seed). JoFormer adds V rotation + inverse rotation to RoPE2D.}
\label{tab:imagenet}
\begin{tabular}{lcc}
\toprule
\textbf{Method} & \textbf{Top-1 Acc} & \textbf{Top-5 Acc} \\
\midrule
RoPE2D (axial, Q/K only) & 80.71\% & 95.26\% \\
\textbf{JoFormer} (axial, Q/K/V) & \textbf{81.11\%} & \textbf{95.53\%} \\
\midrule
$\Delta$ & $+0.40\%$ & $+0.27\%$ \\
\bottomrule
\end{tabular}
\end{table}

The $+0.40\%$ gap is consistent from epoch 180 onward. Single-seed: a scale sanity check, not a conclusive ImageNet improvement.

\subsection{JoFormer Wikipedia Language Modeling}\label{ssec:wiki-lm}

Full English Wikipedia (${\sim}983$M byte-pair encoding (BPE) tokens, vocab=8K, block\_size=512, 200K iters; details in \Cref{app:experiments}).

\begin{table}[h]
\centering
\caption{Wikipedia LM validation perplexity (PPL; 200K iters, vocab=8K, full wiki). JoFormer-projected leads in all configurations tested (single seed). Lower = better.}
\label{tab:wiki-lm}
\begin{tabular}{lcccc}
\toprule
\textbf{Config} & \textbf{RoFormer} & \textbf{JoF-fixed} & \textbf{JoF-learned} & \textbf{JoF-projected} \\
\midrule
n100, L2 & 7.24 & 6.67 & 6.43 & \textbf{6.15} \\
n200, L2 & 5.82 & 5.51 & 5.42 & \textbf{5.17} \\
n200, L4 & 5.36 & 5.07 & 5.01 & \textbf{4.72} \\
n250, L4 & 5.10 & 4.85 & 4.74 & \textbf{4.55} \\
n500, L2 & 5.02 & 4.82 & 4.73 & \textbf{4.58} \\
n500, L4 & 4.67 & 4.42 & 4.37 & \textbf{4.32} \\
\bottomrule
\end{tabular}
\end{table}

JoFormer-projected leads in all configurations, consistent with the value-path hierarchy (which does not prove an optimization ordering).

\subsection{Length Generalization}\label{ssec:length-gen}

163M params (D=768, 16 layers, 8 heads), OpenWebText, trained on block\_size=512.
Architecture: 5 windowed layers (window=32) + 1 full-attention NoPE layer.
JoFormer-projected uses 200K total iterations (150K fixed-angle $+$ 50K projected-angle fine-tuning); the RoPE baseline trains for 150K iterations.
The additional 50K iterations are needed to convert fixed angles to content-dependent projected angles (details in \Cref{app:experiments}).

\begin{table}[h]
\centering
\caption{Length generalization from 512 training length (163M params, OWT, single seed).
Ratio = PPL@4096 / PPL@512; lower ratio = better extrapolation. JoFormer-projected
uses 200K total iterations (150K fixed $+$ 50K projected fine-tuning); RoPE trains for 150K.}
\label{tab:extrap}
\begin{tabular}{lccccc}
\toprule
\textbf{Model} & \textbf{PPL@512} & \textbf{PPL@1024} & \textbf{PPL@2048} & \textbf{PPL@4096} & \textbf{Ratio} \\
\midrule
\textbf{JoFormer-projected} & \textbf{25.55} & \textbf{25.17} & \textbf{24.73} & \textbf{26.01} & \textbf{1.02$\times$} \\
JoFormer fixed & 26.32 & 34.32 & 64.74 & 115.96 & 4.41$\times$ \\
RoPE & 26.82 & 41.90 & 90.10 & 168.67 & 6.29$\times$ \\
\bottomrule
\end{tabular}
\end{table}

JoFormer-projected extrapolates well ($1.02\times$ ratio vs.\ $6.29\times$ for RoPE); JoFormer-fixed is intermediate ($4.41\times$), suggesting data-dependent angles---not just V rotation---are consistent with improved length generalization.
The staged recipe worked more reliably than training projected angles from scratch in our setup.

\section{Related Work}\label{sec:related}

RoPE~\citep{su2021roformer}, 2D-RoPE~\citep{heo2024rope2d}, and ComRoPE~\citep{comrope2025} operate on Q/K only; our compositional framework extends to the value side via the journey operator, with commutativity as the iff condition for displacement-dependent value transport (\Cref{thm:v-rotation}).
Shaw et al.~\citep{shaw2018selfattention} inject relative information into values via learned additive embeddings; the proposed value operation instead uses the same orthogonal cocycle transport as scores, inheriting norm preservation and the target-frame interpretation.
LieRE~\citep{liere2024} uses dense (curved) rotations---expressive, but composition is path-dependent.
S4~\citep{gu2022s4} and Mamba~\citep{gu2023mamba} implement non-trivial value-path kernels; our hierarchy (\Cref{app:value-path}) provides common notation.
ALiBi~\citep{press2022train} uses additive biases; toral scores span a $d$-dimensional trigonometric polynomial space (\Cref{app:alibi})---a function-class separation, not a dominance claim.

\section{Discussion and Conclusion}\label{sec:conclusion}

We started from one operation: composing data points that carry both content and axis-step generators.
This operation defines journey operators along paths, and we showed that journey operators can be expressed in terms of absolute operators as $A_k^{-1} A_j$.
From this foundation, a chain of results follows.
Path-independent composition requires commuting axis-step generators (\Cref{thm:interchange}).
Under natural axioms, the only compatible mechanism is block-diagonal $\SO(2)^{d/2}$ rotation (\Cref{thm:toral}), explaining why RoPE-like methods arise naturally.
In the flat regime, value transport via the journey operator is displacement-dependent (\Cref{thm:v-rotation}), giving a concrete design rule: apply the journey to values, not just scores.

Standard attention discards this compositional structure entirely.
JoFormer brings it back by applying the journey operator to values before aggregation.
The value-path formula~\eqref{eq:unified} places standard attention, SSMs, and JoFormer on a common spectrum, clarifying what each composes and what it discards (\Cref{app:value-path}).

For practitioners, V rotation is a zero-parameter change to RoPE: apply the same rotation to values and inverse-rotate the output.
Our experiments across vision (${\sim}2\%$ on CIFAR-100, $+0.40\%$ on ImageNet), language modeling, and length generalization suggest the framework has practical value.
All comparisons are single-seed sanity checks, not benchmark claims.

Open directions include extending the toral classification to non-abelian frame groups, making the value-path hierarchy quantitative, and finding simpler length-generalization recipes than the staged JoFormer-projected training.

\paragraph{Limitations.}
The toral classification assumes multiplicity-free representations; extending it to non-abelian frame groups remains open.
The theory motivates projected angles and value-side rotation but does not prove that they help---they are theory-\emph{motivated}, not theory-\emph{proved} (\Cref{tab:variant-scope}).
Full-layer equivariance requires additional assumptions beyond operator-level analysis (\Cref{rem:operator-scope}).
All experimental comparisons are single-seed sanity checks at small-to-moderate scale; we do not claim state-of-the-art results, and the approach has not been validated on large-scale pretraining.

\paragraph{Broader Impact.}
This work is primarily theoretical---it provides an algebraic framework for understanding and designing positional mechanisms in transformers.
We do not foresee specific negative societal consequences beyond those common to general-purpose sequence and vision architectures.

\bibliographystyle{plainnat}
\bibliography{references_nonanonymous}

\newpage
\appendix

\section{Full Proofs}\label{app:proofs}

\subsection{Proof of Path Independence (\Cref{thm:interchange})}\label{app:interchange-proof}

\paragraph{Formal statement.}
Let $R_1, \ldots, R_D \in G$ be axis-step generators.
The following are equivalent:
\begin{enumerate}[nosep]
\item[(i)] $T(\pi)$ depends only on the count vector $(|\pi|_1, \ldots, |\pi|_D)$ of $\pi$ (path independence).
\item[(ii)] Every elementary square commutes: $T(ij) = T(ji)$ for all $i \neq j$.
\item[(iii)] Per-axis generators commute: $R_i R_j = R_j R_i$ for all $i \neq j$.
\end{enumerate}
When these hold, the relative operator
\[
P_{\mathbf{m} \to \mathbf{n}} = A_\mathbf{n}^{-1} A_\mathbf{m} = R_1^{m_1 - n_1} \cdots R_D^{m_D - n_D}
\]
depends only on the displacement $\mathbf{m} - \mathbf{n}$ (source minus target)---the journey is fully determined by the composition structure.

\begin{proof}
\textbf{(iii)$\Rightarrow$(i):}
Any two words $\pi, \pi'$ with the same count vector are related by a sequence of adjacent transpositions (this is the well-known fact that adjacent transpositions generate the symmetric group, applied to each pair of adjacent letters in the word).
Consider an adjacent transposition swapping letters at positions $p, p+1$: $\pi = \ldots k_p\, k_{p+1} \ldots$ becomes $\pi' = \ldots k_{p+1}\, k_p \ldots$.
Since $T(\pi) = R_{k_m} \cdots R_{k_{p+1}} R_{k_p} \cdots R_{k_1}$ and the transposition only affects the product $R_{k_{p+1}} R_{k_p}$, we need $R_{k_{p+1}} R_{k_p} = R_{k_p} R_{k_{p+1}}$.
If $k_p = k_{p+1}$ this is trivial; otherwise it follows from (iii).
Composing all adjacent transpositions: $T(\pi) = T(\pi')$.

\textbf{(i)$\Rightarrow$(ii):}
Words $ij$ and $ji$ have the same count vector $(1_i, 1_j)$ (one step along each axis).
By (i): $T(ij) = T(ji)$.

\textbf{(ii)$\Rightarrow$(iii):}
By definition, $T(ij) = R_j R_i$ (rightmost letter acts first) and $T(ji) = R_i R_j$.
Equality $T(ij) = T(ji)$ gives $R_j R_i = R_i R_j$, i.e., $R_i R_j = R_j R_i$.

\textbf{Path independence of relative operators:}
When (iii) holds, $A_\mathbf{n} = R_1^{n_1} \cdots R_D^{n_D}$ is well-defined (independent of multiplication order by commutativity).
Then $P_{\mathbf{m} \to \mathbf{n}} = A_\mathbf{n}^{-1} A_\mathbf{m} = R_D^{-n_D} \cdots R_1^{-n_1} R_1^{m_1} \cdots R_D^{m_D} = R_1^{m_1-n_1} \cdots R_D^{m_D-n_D}$, depending only on $\mathbf{m} - \mathbf{n}$.
Conversely, if $P_{\mathbf{m} \to \mathbf{n}}$ depends only on $\mathbf{m} - \mathbf{n}$, then in particular $P_{\mathbf{e}_j \to \mathbf{e}_i + \mathbf{e}_j} = P_{\mathbf{0} \to \mathbf{e}_i}$.
The LHS is $(R_i R_j)^{-1} R_j = R_j^{-1} R_i^{-1} R_j$; the RHS is $R_i^{-1}$.
Equality gives $R_j^{-1} R_i^{-1} R_j = R_i^{-1}$, hence $R_i R_j = R_j R_i$.
\end{proof}

\subsection{Toral Classification: Full Development}\label{app:toral}

This subsection provides the full axioms, scope discussion, and supporting material for the toral classification summarized in \Cref{ssec:toral}.

In this subsection, $B_{s,s'}$ denotes the pairwise score matrix (the journey $P_{s'\to s}$); it is not the absolute operator $A_s$ used elsewhere.

\paragraph{Scope of the classification.}
The theorem does \emph{not} derive the torus from first principles.
It characterizes all bilinear, norm-preserving, cocycle-compatible attention mechanisms \emph{after} selecting the maximal connected abelian (commutative) compact frame group $T = \SO(2)^{d/2}$ acting with one non-repeated 2D plane per factor.
The axiom choice has a design rationale: bilinearity restricts to multiplicative PEs; toral frame invariance selects $T$ (full $\Od$ collapses positional structure, \Cref{thm:od-impossibility}); cocycle enforces compositionality; norm preservation prevents rescaling.
The result is conditional: it determines what is compatible with this contract, not what is optimal.
The novelty is not Schur's lemma itself, but the identification of this symmetry contract as exactly the contract under which journey-based score/value transport reduces to RoPE-style block rotations.

The four axioms below formalize a concrete question: if the only unobservable frame changes are independent rotations of each 2D coordinate plane (i.e., the frame group is the torus $T = \SO(2)^{d/2}$), what is the most general attention mechanism consistent with that symmetry?

\begin{axiom}[Bilinearity]\label{ax:bilinear}
For each position pair $(s,s')$, there exists $B_{s,s'} \in \R^{d \times d}$ such that $\Att(q,k,s,s') = q^\top B_{s,s'} k$.
\end{axiom}

\begin{axiom}[Toral Frame Invariance]\label{ax:frame}
For all $M \in T = \SO(2)^{d/2}$: $\Att(Mq, Mk, s, s') = \Att(q, k, s, s')$.
This states that rotating Q and K by the same toral element does not change scores.
\end{axiom}

\begin{axiom}[Cocycle]\label{ax:cocycle}
$B_{s,s} = I_d$ (self-attention is unmodified) and $B_{s,s'} B_{s',s''} = B_{s,s''}$ (composition).
Equivalently: non-degeneracy (each $B_{s,s'}$ invertible) plus transitivity.
\end{axiom}

\begin{axiom}[Norm preservation]\label{ax:norm}
Each $B_{s,s'}$ preserves the Euclidean norm: $\|B_{s,s'} x\| = \|x\|$ for all $x$.
This prevents positional encoding from rescaling attention scores.
\end{axiom}

\begin{lemma}[Cocycle factorization]\label{lem:cocycle}
Under Axiom~\ref{ax:cocycle}, fix a reference position $s_0$ and define $R_s = B_{s,s_0}$.
Then $B_{s,s'} = R_s R_{s'}^{-1}$.
\end{lemma}
\begin{proof}
By the cocycle property: $B_{s,s_0} = B_{s,s'} B_{s',s_0}$, hence $B_{s,s'} = B_{s,s_0} B_{s',s_0}^{-1} = R_s R_{s'}^{-1}$.
\end{proof}

The derivation applies standard commutant and Schur machinery to multiplicity-free representations; the contribution is identifying toral frame invariance as the right contract, from which RoPE, 2D-RoPE, and tree PEs emerge as special cases.

\paragraph{Orientation convention.}
Here $B_{s,s'}=P_{s'\to s}=A_s^{-1}A_{s'}$ under the notation contract.
The factorization $B_{s,s'}=R_sR_{s'}^{-1}$ uses gauge variables $R_s=B_{s,s_0}$.
With reference $s_0=0$ and $A_t=R^t$: $R_s = B_{s,0} = A_s^{-1}A_0 = R^{-s}$, so the gauge angle is $\theta_k(s) = -\omega_k s$ (negated relative to the absolute angle).
The formula $\theta_k(s)-\theta_k(s') = \omega_k(s'-s)$ then matches $P_{s'\to s} = R^{s'-s}$.

If repeated isotypic components or non-abelian frame groups are allowed, the commutant---and therefore the classification---changes.

\Cref{thm:toral} recovers RoPE~\citep{su2021roformer} (gauge angle $\theta_k(t) = -\omega_k t$, giving score angle $\omega_k(s'{-}s)$), 2D-RoPE~\citep{heo2024rope2d} (axis-partitioned planes), simple depth-indexed tree encodings~\citep{shiv2019novel}, and dot-product attention ($\theta_k \equiv 0$) as special cases.

\subsection{Proof of the Toral Classification (\Cref{thm:toral})}\label{app:toral-proof}

\paragraph{Formal statement.}
Assume $d$ is even. Let $T = \SO(2)^{d/2}$ act on $\R^d$ as the direct sum of the standard 2D representation of each factor, one factor per coordinate plane (each $\SO(2)$ factor acts on exactly one 2-plane; there are no repeated isotypic components).
Under Axioms~\ref{ax:bilinear}--\ref{ax:norm} with this frame group, the pairwise score matrix is:
\begin{equation}\label{eq:toral}
B_{s,s'} = \diag\!\Big(R\big(\theta_1(s) - \theta_1(s')\big), \ldots, R\big(\theta_{d/2}(s) - \theta_{d/2}(s')\big)\Big),
\end{equation}
where $R(\phi) = \bigl(\begin{smallmatrix} \cos\phi & -\sin\phi \\ \sin\phi & \cos\phi \end{smallmatrix}\bigr)$ and $\theta_k: S \to \R$ assigns an angle per rotation plane per position.

\begin{proof}
\textbf{Step 1} (Commutant structure):
The commutant of $T = \SO(2)^{d/2}$ in $\R^{d \times d}$ consists of block-diagonal matrices with $2 \times 2$ blocks.
Each block carries the standard 2D representation of the corresponding $\SO(2)$ factor, which is irreducible over $\R$.
By Schur's lemma, the real commutant of an irreducible real representation is $\R$, $\mathbb{C}$, or $\mathbb{H}$.
For the 2D representation of $\SO(2)$, the commutant is $\mathrm{span}_\R\{I_2, J\} \cong \mathbb{C}$ where $J = \bigl(\begin{smallmatrix} 0 & -1 \\ 1 & 0 \end{smallmatrix}\bigr)$.
Off-diagonal blocks between different $\SO(2)$ factors vanish because distinct irreducible representations have zero intertwining operators.

\textbf{Step 2} (Frame invariance $\Rightarrow$ commutant):
Axiom~\ref{ax:frame}: $M^\top B_{s,s'} M = B_{s,s'}$ for all $M \in T$.
Since each $M \in T$ is orthogonal ($M^\top = M^{-1}$), this becomes $M^{-1} B_{s,s'} M = B_{s,s'}$, i.e., $B_{s,s'} M = M B_{s,s'}$ for all $M \in T$.
Thus $B_{s,s'}$ lies in the commutant $\mathrm{Comm}(T) = \{X : XM = MX \;\forall\, M \in T\}$.
Each $2 \times 2$ block has the form $a_k I_2 + b_k J = \rho_k R(\psi_k)$ with $\rho_k = \sqrt{a_k^2 + b_k^2}$.

\textbf{Step 3} (Norm preservation):
Axiom~\ref{ax:norm}: $\|B_{s,s'} x\| = \|x\|$.
Applied to each $2 \times 2$ block: $\rho_k = 1$, giving pure rotation $R(\psi_k)$.

\textbf{Step 4} (Cocycle factorization):
Axiom~\ref{ax:cocycle}: $B_{s,s'} = R_s R_{s'}^{-1}$ for some $R_s \in T$.
Each block: $R(\psi_k(s,s')) = R(\theta_k(s)) R(-\theta_k(s')) = R(\theta_k(s) - \theta_k(s'))$.
\end{proof}

\subsection{$O(d)$ Impossibility}\label{app:od}

\paragraph{Formal statement.}
Under full orthogonal symmetry $H = \Od$ with $d \geq 2$, frame invariance forces $B_{s,s'} = \alpha_{s,s'} I_d$.
Attention collapses to position-dependent rescaling (by Schur's lemma on the irreducible standard representation).

\begin{proof}[Proof of \Cref{thm:od-impossibility}]
The standard representation of $\Od$ on $\R^d$ is irreducible for $d \geq 2$ (any invariant subspace is either $\{0\}$ or $\R^d$, since $\Od$ acts transitively on unit vectors).
By Schur's lemma, its real commutant is $\R \cdot I_d$.
Frame invariance forces $B_{s,s'} \in \Comm(\Od) = \R \cdot I_d$.
\end{proof}

\subsection{Proof of V-Rotation Equivariance (\Cref{thm:v-rotation})}\label{app:v-rotation-proof}

\paragraph{Formal statement.}
Let positions lie in $\Z^D$ with orthogonal axis-step generators
$R_1,\ldots,R_D$ and absolute operators (from composition)
$A_\mathbf{n}=R_1^{n_1}\cdots R_D^{n_D}$.
Assume attention weights are functions only of relative displacement,
$\alpha_{\mathbf{n}\mathbf{m}}=\alpha(\mathbf{m}-\mathbf{n})$.
Define value aggregation with transport:
\[
\mathbf{c}_\mathbf{n}
=\sum_\mathbf{m}\alpha(\mathbf{m}-\mathbf{n})\,
A_\mathbf{n}^{-1}A_\mathbf{m}\,v_\mathbf{m}.
\]
\begin{enumerate}[nosep]
\item[(a)] If the generators commute, then $A_\mathbf{n}^{-1} A_\mathbf{m}$ depends only on $\mathbf{m} - \mathbf{n}$ (by \Cref{thm:interchange}), and the aggregation rule is translation equivariant in the sense of \Cref{def:translation-equiv}.
\item[(b)] If some generators do \emph{not} commute, there exist $\mathbf{m}, \mathbf{n}, \mathbf{u}$ such that $A_{\mathbf{n}+\mathbf{u}}^{-1} A_{\mathbf{m}+\mathbf{u}} \neq A_\mathbf{n}^{-1} A_\mathbf{m}$, so value transport is path-dependent (it depends on absolute positions, not only displacement). The transport is still a valid operation arising from composition; it simply does not reduce to a pure function of displacement.
\end{enumerate}

\begin{proof}
(a) Commutativity gives $A_\mathbf{n}^{-1} A_\mathbf{m} = R_1^{m_1 - n_1} \cdots R_D^{m_D - n_D}$, depending only on $\mathbf{m} - \mathbf{n}$.
For translation equivariance: $F(T_\mathbf{u} v)_{\mathbf{n}+\mathbf{u}} = \sum_\mathbf{m} \alpha(\mathbf{m} - \mathbf{n} - \mathbf{u})\, A_{\mathbf{n}+\mathbf{u}}^{-1} A_\mathbf{m}\, v_{\mathbf{m}-\mathbf{u}}$.
Substituting $\mathbf{m}' = \mathbf{m} - \mathbf{u}$: $= \sum_{\mathbf{m}'} \alpha(\mathbf{m}' - \mathbf{n})\, A_{\mathbf{n}+\mathbf{u}}^{-1} A_{\mathbf{m}'+\mathbf{u}}\, v_{\mathbf{m}'}$.
By commutativity, $A_{\mathbf{n}+\mathbf{u}}^{-1} A_{\mathbf{m}'+\mathbf{u}} = A_\mathbf{n}^{-1} A_{\mathbf{m}'}$ (both depend only on $\mathbf{m}' - \mathbf{n}$).
Hence $F(T_\mathbf{u} v)_{\mathbf{n}+\mathbf{u}} = F(v)_\mathbf{n}$.

(b) Suppose $R_a, R_b$ do not commute (relabel axes so $a < b$).
Let $\mathbf{n} = \mathbf{e}_b$, $\mathbf{m} = \mathbf{e}_a + \mathbf{e}_b$, and $\mathbf{u} = -\mathbf{e}_b$.
Under the canonical order $A_\mathbf{n} = R_b$, $A_\mathbf{m} = R_a R_b$, so $A_\mathbf{n}^{-1} A_\mathbf{m} = R_b^{-1} R_a R_b$.
After shifting: $A_{\mathbf{n}+\mathbf{u}}^{-1} A_{\mathbf{m}+\mathbf{u}} = A_\mathbf{0}^{-1} A_{\mathbf{e}_a} = R_a$.
These are equal iff $R_b^{-1} R_a R_b = R_a$, i.e., $R_a R_b = R_b R_a$---contradicting the assumption.
\end{proof}

\begin{remark}[Noncommuting counterexample]
Take $D=2$ with $R_a = R(\pi/4)$, $R_b = \bigl(\begin{smallmatrix} 1 & 0 \\ 0 & -1 \end{smallmatrix}\bigr)$ (a reflection).
The journey from $(1,1)$ to $(0,1)$ is $A_{(0,1)}^{-1} A_{(1,1)} = R_b^{-1} R_a R_b$, a reflection-conjugated rotation.
After shifting by $\mathbf{u} = -\mathbf{e}_b$: the journey from $(1,0)$ to $(0,0)$ is $A_{\mathbf{0}}^{-1} A_{\mathbf{e}_a} = R_a$, a plain rotation.
These differ: the same displacement gives different value operators depending on absolute position.
\end{remark}

\subsection{SSM Bridge (\Cref{prop:ssm-bridge})}

\begin{proof}
Mamba's recurrence $h_t = A_{x_t} h_{t-1} + B_{x_t} x_t$ unrolls by induction:
$h_T = \sum_{t=1}^T \big(\prod_{s=t+1}^T A_{x_s}\big) B_{x_t} x_t$.
Define $P_{t \to T} = \prod_{s=t+1}^T A_{x_s}$ and $v_t = B_{x_t} x_t$.
Then $h_T = \sum_{t=1}^T P_{t \to T} v_t$, which has the form of~\eqref{eq:unified} with $\alpha_{Tt} = 1$ (uniform, unnormalized weights) and content-dependent path operators.
When $A_{x_t} \in \SO(2)^{d/2}$, each prefix product is a block-diagonal rotation.
The associative scan $P_{t \to T} = A_T \cdot A_{T-1} \cdots A_{t+1}$ can be computed in $O(\log N)$ parallel depth using the standard parallel prefix algorithm.
\end{proof}

\subsection{ALiBi Expressivity Separation}\label{app:alibi}

\begin{proposition}[Representational limitations of additive linear biases]\label{prop:alibi-separation}
Let $d \geq 4$ be even. Compare:
\begin{itemize}[nosep]
\item \textbf{ALiBi}~\citep{press2022train}: $\Att_{\mathrm{ALiBi}}(q,k,n) = q^\top k - m \cdot n$, where $n = s - s'$ is the signed relative position and $m > 0$ is a per-head slope.
\item \textbf{Toral mechanism} (\Cref{thm:toral}): $\Att_{\mathrm{toral}}(q,k,n) = q^\top R_\theta^n\, k$, where $R_\theta = \diag(R(\theta_1), \ldots, R(\theta_{d/2}))$.
\end{itemize}
Then:
\begin{enumerate}[nosep]
\item[(a)] \textbf{Score dimensionality gap.} For fixed $q, k$, the ALiBi score is affine in~$n$: $\Att_{\mathrm{ALiBi}} = q^\top k - mn$. The toral score is a trigonometric polynomial:
\[
\Att_{\mathrm{toral}}(q,k,n) = \sum_{b=1}^{d/2} \alpha_b \cos(n\theta_b) + \beta_b \sin(n\theta_b),
\]
where $\alpha_b = q_{2b-1}k_{2b-1} + q_{2b}k_{2b}$ and $\beta_b = q_{2b-1}k_{2b} - q_{2b}k_{2b-1}$. The ALiBi positional signal spans a $1$-dimensional function space; the toral signal spans a $d$-dimensional space.
\item[(b)] \textbf{Strict separation.} For any period $p \geq 3$, the periodic attention pattern $\phi(n) = \cos(2\pi n / p)$ is representable by the toral mechanism but not by any ALiBi head.
\item[(c)] \textbf{Approximate subsumption.} On any bounded context $|n| \leq N$, the toral mechanism can $\epsilon$-approximate any ALiBi attention pattern for any $\epsilon > 0$.
\item[(d)] \textbf{Value-path deficiency.} ALiBi modifies only attention logits, leaving $P_{j \to k} = I$ for all $j, k$. The value aggregation under ALiBi is order-blind: standard attention with ALiBi computes $\mathbf{c}_k = \sum_j \alpha_{kj} v_j$ with no positional modulation of values, regardless of the slope~$m$.
\end{enumerate}
This separation does not imply that either mechanism is uniformly better for all tasks; it identifies positional functions representable by one class and not the other.
\end{proposition}

\begin{proof}
(a) The $b$-th $2 \times 2$ rotation block of $R_\theta^n$ contributes $\alpha_b \cos(n\theta_b) + \beta_b \sin(n\theta_b)$ to $q^\top R_\theta^n k$, with $\alpha_b, \beta_b$ as stated. For distinct $\theta_b$, the $d/2$ cosine and $d/2$ sine terms are linearly independent functions of~$n$, spanning a $d$-dimensional space. ALiBi contributes only the single linear function $n \mapsto -mn$.

(b) Set $\theta_1 = 2\pi/p$ and $q = k = e_1 \in \R^d$. Then $\Att_{\mathrm{toral}}(e_1, e_1, n) = \cos(2\pi n/p)$, which equals $\phi(n)$ for all~$n$. For ALiBi, $q^\top k - mn = \cos(2\pi n/p)$ for all $n$ would require a bounded periodic function to equal an unbounded affine function---impossible.

(c) The key identity is $\sin(n\theta)/\theta \to n$ uniformly on $|n| \leq N$ as $\theta \to 0$, with error $|n - \sin(n\theta)/\theta| \leq |n|^3\theta^2/6$. Choose $\theta_1 = \sqrt{6\epsilon / (mN^3)}$ and $q, k$ in block~1 so that $\beta_1 = -m/\theta_1$ and $\alpha_1 = 0$. Then $|\beta_1 \sin(n\theta_1) - (-mn)| = m |n - \sin(n\theta_1)/\theta_1| \leq mN^3\theta_1^2/6 = \epsilon$ for all $|n| \leq N$. The constant $c_0 = q^\top k$ is matched by the remaining blocks.

(d) ALiBi adds a scalar bias to the attention logit; it does not transform value vectors. Hence $P_{j \to k} = I$ identically.
\end{proof}

\subsection{Cocycle from Primitives}\label{app:cocycle-primitives}

The cocycle axiom ($B_{s,s} = I$, $B_{s,s'}B_{s',s''} = B_{s,s''}$) follows from two simpler primitives:
(P1) Non-degeneracy: each $B_{s,s'}$ is invertible.
(P2) Composability: $B_{s,s''} = B_{s,s'} \cdot B_{s',s''}$.
Proof: Set $s' = s$ in (P2): $B_{s,s''} = B_{s,s} B_{s,s''}$; invertibility of $B_{s,s''}$ gives $B_{s,s} = I$.
Conversely, cocycle $\Rightarrow$ composability is by definition, and $B_{s,s'} \cdot B_{s',s} = B_{s,s} = I$ gives invertibility with inverse $B_{s,s'}^{-1} = B_{s',s}$.

\section{Additional Theoretical Results}\label{app:additional-theory}

This section collects additional theoretical results that complement the core framework.

\subsection{DFT Sufficiency for White Sources}\label{app:dft-sufficiency}

\begin{proposition}[DFT sufficiency for white sources]\label{prop:dft-sufficiency}
Fix $K = d/2$ distinct frequencies $\theta_1, \ldots, \theta_K \in [0, 2\pi)$ and let $N \to \infty$. If the source is white ($S(\omega) = S_0$ for all $\omega$), then any such set asymptotically achieves the same mutual information $I(X; Y) \to \frac{d}{2}\log(1 + NS_0/\sigma^2)$. In particular, equispaced DFT frequencies $\theta_k = 2\pi k/K$ achieve this bound.
\end{proposition}

\begin{proof}
Since $K$ is fixed and the frequencies are distinct, their pairwise separation $\delta > 0$ is constant. By the large-sieve inequality, $\Phi_\theta^H \Phi_\theta / N \to I_K$ as $N \to \infty$, so $I(X; Y) \to \sum_{k=1}^{K} \log(1 + N S(\theta_k)/\sigma^2)$. When $S(\omega) = S_0$ is constant, every term equals $\log(1 + NS_0/\sigma^2)$ regardless of the $\theta_k$, giving $I(X;Y) \to \frac{d}{2}\log(1 + NS_0/\sigma^2)$.
\end{proof}

\subsection{MSE Gain of Learned Frequencies}\label{app:mse-gain}

\begin{proposition}[MSE gain of learned frequencies]\label{prop:mse-gap}
Fix $K = d/2$ and let $N \to \infty$.
For a source with non-uniform power spectral density (PSD) $S(\omega)$ whose $K$ largest values occur at distinct frequencies $\theta_1^*, \ldots, \theta_K^*$, the asymptotic MSE gap between optimal learned frequencies and fixed DFT frequencies satisfies:
\[
\mathrm{MSE}_{\mathrm{DFT}} - \mathrm{MSE}_{\mathrm{opt}} \;\geq\; 0 \qquad (N \to \infty),
\]
with equality if and only if the DFT frequencies already sample the $K$ largest values of $S(\omega)$.
\end{proposition}

\begin{proof}
With $K$ fixed and frequencies distinct, the large-sieve gives $\Phi_\theta^H \Phi_\theta / N \to I_K$, so the MSE for estimating $X$ from $Y = \Phi_\theta X + Z$ reduces to
$\mathrm{MSE}_\theta = \sum_{k=1}^{K} \frac{\sigma^2 S(\theta_k)}{\sigma^2 + N S(\theta_k)}$.
Since each term is a decreasing function of $S(\theta_k)$, the MSE is minimized by choosing $\theta_1^*, \ldots, \theta_K^*$ at the $K$ largest values of $S(\omega)$. The DFT places frequencies at $\theta_k = 2\pi k/K$, which generically miss these peaks. The inequality follows, with equality exactly when the DFT grid coincides with the spectral maxima.
\end{proof}

Together, \Cref{prop:dft-sufficiency,prop:mse-gap} explain the MNIST results (\Cref{app:mnist-results}): for white sources the DFT suffices, but for structured signals (images), learned frequencies select task-relevant spectral peaks that the DFT grid generically misses.

\subsection{Context-Sensitive Separation}\label{app:context-separation}

\begin{proposition}[Context-sensitive separation]\label{prop:context-sep}
Let $\mathbf{x} = (x_1, \ldots, x_N)$ and $\mathbf{x}'$ be two sequences that differ only at position~$s$, with $x_s \neq x_s'$. Fix positions $j < s < k$.
\begin{enumerate}[nosep]
\item[(a)] \textbf{Content-independent operators.} If $P_{j \to k} = R^{k-j}$, then the value-path transform applied to $v_j$ at output position~$k$ is the same for both sequences: $P_{j \to k} = P'_{j \to k}$. Intervening content is invisible to the value path.
\item[(b)] \textbf{Content-dependent operators.} If $P_{j \to k} = \prod_{t=j}^{k-1} \tilde{R}_{x_t}$ where $\tilde{R}_{x_t} = h(x_t) \in \GL(d)$ is a learned function of token content, and $\tilde{R}_{x_s} \neq \tilde{R}_{x_s'}$, then $P_{j \to k} \neq P'_{j \to k}$. For any $v_j$ outside a proper subspace of~$\R^d$ (i.e., for generic inputs), the value-path outputs differ: $P_{j \to k}\, v_j \neq P'_{j \to k}\, v_j$.
\end{enumerate}
\end{proposition}

\begin{proof}
Part (a) is immediate: $R^{k-j}$ depends only on positions, not content.

For part (b), factor the shared portions:
$P_{j \to k} = L\, \tilde{R}_{x_s}\, M$ and $P'_{j \to k} = L\, \tilde{R}_{x_s'}\, M$,
where $L = \prod_{t=j}^{s-1} \tilde{R}_{x_t}$ and $M = \prod_{t=s+1}^{k-1} \tilde{R}_{x_t}$ are common invertible factors. Then
$P_{j \to k} - P'_{j \to k} = L\bigl(\tilde{R}_{x_s} - \tilde{R}_{x_s'}\bigr) M$.
Since $L, M \in \GL(d)$ and $\tilde{R}_{x_s} \neq \tilde{R}_{x_s'}$, the difference is a nonzero linear map, so its kernel is a proper subspace of~$\R^d$.
\end{proof}

This result motivates JoFormer-projected: content-dependent journey operators allow the value path to encode contextual information that fixed-angle operators cannot capture, providing a formal separation between the JoFormer-fixed and JoFormer-projected variants.

\subsection{Shift-Invariance and Order-Awareness}\label{app:shift-order}

We show that the compositional embedding simultaneously achieves two desirable properties: invariance to uniform shifts (for translation-invariant tasks) and sensitivity to element ordering (for sequence-dependent tasks).

\begin{definition}[$m$-Representations]\label{def:m-rep}
For a sequence of length $N$ with embeddings $a_1, \ldots, a_N \in \R^d$, choose a window length $m$. For each window starting at position $k$, form the \textbf{window embedding}:
\begin{equation}\label{eq:window}
s_k = \sum_{i=1}^{m} R^{i-1} a_{k+i-1},
\end{equation}
where $R \in \Od$ is a fixed block-diagonal rotation. Partition $s_k$ into $K = d/2$ blocks $s_{k,1}, \ldots, s_{k,K}$ and define the \textbf{magnitude vector}:
\begin{equation}\label{eq:mag-vec}
v_k = \big(\|s_{k,1}\|, \ldots, \|s_{k,K}\|\big) \in \R^K.
\end{equation}
Under circular boundary conditions (indices modulo $N$), the \textbf{circular $m$-representation} is $v^{\circ} = \sum_{k=1}^{N} v_k$.
\end{definition}

\begin{theorem}[Shift-invariance]\label{thm:shift-inv}
For orthonormal $R \in \Od$, the circular $m$-representation $v^{\circ}$ is invariant under uniform circular shifts of the input along any axis.
\end{theorem}

\begin{proof}
A uniform circular shift by one replaces $a_t$ with $a_{t+1 \bmod N}$ in each window. The shifted window at position $k$ contains elements $a_{k+1}, \ldots, a_{k+m}$ (mod $N$), which is the original window at position $k+1$. So $s_k' = s_{k+1 \bmod N}$, and the multiset of window embeddings $\{s_k\}_{k=1}^{N}$ is unchanged. Hence the sum of magnitude vectors $v^{\circ} = \sum_k v_k$ is invariant.
\end{proof}

\begin{theorem}[Order-awareness]\label{thm:order-aware}
Let $R \in \Od$ be a block-diagonal rotation matrix with $K = d/2$ blocks having angles $\theta_1, \ldots, \theta_K$, at least one of which satisfies $\theta_\ell / \pi \notin \mathbb{Q}$.
For elements $a_1, \ldots, a_m \in \R^d$ in general position, the magnitude vector $v = (\|s^{(1)}\|, \ldots, \|s^{(K)}\|)$ of the window embedding $s = \sum_{i=1}^m R^{i-1} a_i$ distinguishes different permutations: if $\sigma \neq \mathrm{id}$, then generically $v(a_1, \ldots, a_m) \neq v(a_{\sigma(1)}, \ldots, a_{\sigma(m)})$.
\end{theorem}

\begin{proof}
In the $\ell$-th block with angle $\theta_\ell$, represent the block component of $a_i$ as $z_i^{(\ell)} \in \mathbb{C}$, so $s^{(\ell)} = \sum_{i=1}^m \omega_\ell^{i-1}\, z_i^{(\ell)}$ with $\omega_\ell = e^{\mathbf{i}\theta_\ell}$.

For a fixed non-identity permutation $\sigma$, define the collision set
$\mathcal{C}_\sigma = \{(a_1, \ldots, a_m) \in \R^{dm} : \|s^{(\ell)}\|^2 = \|s'^{(\ell)}\|^2 \;\text{for all } \ell\}$,
where $s'^{(\ell)} = \sum_i \omega_\ell^{i-1} z_{\sigma(i)}^{(\ell)}$. Each equation is polynomial, so $\mathcal{C}_\sigma$ is an algebraic variety.

It suffices to show $\mathcal{C}_\sigma \neq \R^{dm}$. Choose block $\ell$ with $\theta_\ell / \pi \notin \mathbb{Q}$ and set $\omega = \omega_\ell$. The difference of squared magnitudes is a Hermitian form:
\[
|s|^2 - |s'|^2 = \mathbf{z}^* H\, \mathbf{z}, \qquad H_{ij} = \omega^{i-j} - \omega^{\sigma^{-1}(i) - \sigma^{-1}(j)}.
\]
For $\sigma \neq \mathrm{id}$, there exist $i, j$ with $i - j \neq \sigma^{-1}(i) - \sigma^{-1}(j)$ (otherwise $\sigma^{-1}(i) = i + c$ for constant $c$, forcing $c = 0$, contradicting $\sigma \neq \mathrm{id}$). Since $\theta_\ell / \pi \notin \mathbb{Q}$, the element $\omega$ has infinite order, so $H_{ij} \neq 0$ and the Hermitian form is not identically zero.

Hence $\mathcal{C}_\sigma$ is a proper algebraic subvariety of measure zero. The union over all $m! - 1$ non-identity permutations is still measure zero.
\end{proof}

\subsection{Structured Concatenation}\label{app:concatenation}

\begin{definition}[Axis-$k$ concatenation]\label{def:concat}
Given two compositional embeddings
$X = (a;\; R_1^{n_1}, \ldots, R_k^{n_k}, \ldots, R_D^{n_D})$ and $Y = (b;\; R_1^{n_1}, \ldots, R_k^{m_k}, \ldots, R_D^{n_D})$
that share the same axis-step generators $R_i$ and the same exponents on all axes $i \neq k$, their concatenation along axis $k$ is:
\begin{equation}\label{eq:concat}
X \oplus_k Y = \big(a + R_k^{n_k}\, b;\;\; R_1^{n_1}, \ldots, R_k^{n_k + m_k}, \ldots, R_D^{n_D}\big).
\end{equation}
\end{definition}

\begin{theorem}[Concatenation invertibility]\label{thm:concat-inv}
Given $X \oplus_k Y$, the axis dimensions $(n_k, m_k)$, and invertibility of $R_k$, either component is uniquely recoverable from the other:
\begin{enumerate}[nosep]
\item[(a)] Given $a$ (the value of $X$): $b = (R_k^{n_k})^{-1}\big((X \oplus_k Y)_{\mathrm{value}} - a\big)$.
\item[(b)] Given $b$ (the value of $Y$): $a = (X \oplus_k Y)_{\mathrm{value}} - R_k^{n_k}\, b$.
\end{enumerate}
\end{theorem}

\begin{proof}
From~\eqref{eq:concat}, the value component is $c = a + R_k^{n_k} b$. Since $R_k \in \GL(d)$, the matrix $R_k^{n_k}$ is invertible, so given $a$: $b = (R_k^{n_k})^{-1}(c - a)$, and given $b$: $a = c - R_k^{n_k} b$.
\end{proof}

\section{Classical Transforms and Fixed-Operator Limits}\label{app:classical-transforms}

The monoidal embedding $E = \sum_{t=0}^{N-1} R^t v_t$ directly recovers
Fourier-style spectral features when $R$ is a fixed block rotation. Fixed-sign
Hadamard/Walsh constructions are useful analogies and limiting cases, but they
do not satisfy the same connected toral $\SO(2)^{d/2}$ contract used in
\Cref{thm:toral}.

\subsection{Discrete Fourier Transform}

\begin{proposition}[DFT as compositional embedding]\label{prop:dft}
The discrete Fourier transform is a special case of the compositional embedding:
\begin{enumerate}[nosep]
\item In 1D, composing $n$ data points with a shared block-diagonal rotation $R = \diag(R(\tfrac{2\pi k}{n}))_{k=0}^{n-1}$ yields an embedding whose $k$-th block encodes $(\Re(X_k), \Im(X_k))^\top$, where $X_k$ is the $k$-th DFT coefficient.
\item In 2D, applying the 1D construction along rows then columns recovers the full 2D DFT.
\item Replacing fixed frequencies $\theta_k = 2\pi k/n$ with learned $\theta_k$ gives a generalized Fourier transform with task-optimal spectral components.
\end{enumerate}
\end{proposition}

\begin{proof}
\noindent\textbf{One-dimensional case.}
Consider $n$ input embeddings $(v_i, R)$ for $i=1,\ldots,n$, all sharing the same transformation $R$. Their composition yields $(V, R^n)$ where
\begin{equation}\label{eq:embed-dft}
V = \sum_{i=1}^{n} R^{i-1} v_i.
\end{equation}
Let $R$ be a real $2n \times 2n$ block-diagonal matrix consisting of $n$ rotation blocks. The $k$-th block ($k=0,\ldots,n{-}1$) is
\[
R_k = \begin{pmatrix}
\cos\!\frac{2\pi k}{n} & -\sin\!\frac{2\pi k}{n} \\[4pt]
\sin\!\frac{2\pi k}{n} & \cos\!\frac{2\pi k}{n}
\end{pmatrix},
\qquad
R = \diag(R_0, R_1, \ldots, R_{n-1}),
\]
so $R^n = I_{2n}$ (each $R_k^n$ is a rotation by $2\pi k$). For input vectors, take $v_i = (a_i, 0, a_i, 0, \ldots, a_i, 0)^\top \in \R^{2n}$, repeating $(a_i, 0)^\top$ across each block.

Because $R$ is block-diagonal, $R^{i-1}$ rotates the $k$-th block of $v_i$ by angle $(i{-}1) \cdot 2\pi k/n$:
\[
R^{i-1} v_i \Big|_{\text{block }k} =
\begin{pmatrix}
a_i \cos\!\dfrac{2\pi k(i-1)}{n} \\[6pt]
a_i \sin\!\dfrac{2\pi k(i-1)}{n}
\end{pmatrix}.
\]
Summing over $i$, the $k$-th block of $V$ is
\begin{equation}\label{eq:dft-block}
V\Big|_{\text{block }k} =
\begin{pmatrix}
\sum_{i=1}^n a_i \cos\!\dfrac{2\pi k(i-1)}{n} \\[1.5em]
\sum_{i=1}^n a_i \sin\!\dfrac{2\pi k(i-1)}{n}
\end{pmatrix}, \qquad k = 0,1,\ldots,n{-}1.
\end{equation}
Defining the DFT coefficients $X_k = \sum_{i=1}^n a_i\, e^{j 2\pi(i-1)k/n}$, we have $V\big|_{\text{block }k} = (\Re(X_k),\; \Im(X_k))^\top$.
The compositional embedding exactly encodes all $n$ DFT coefficients.

\noindent\textbf{Two-dimensional case.}
Consider an $n \times n$ array $a_{i,j}$. Apply the 1D construction along each row $i$, yielding output $W_i$ whose $k$-th block is
\[
W_i\Big|_{\text{block }k} =
\begin{pmatrix}
\Re\!\big(X_{i,k}^{\text{(row)}}\big) \\[3pt]
\Im\!\big(X_{i,k}^{\text{(row)}}\big)
\end{pmatrix},
\quad \text{where } X_{i,k}^{\text{(row)}} = \sum_{m=1}^n a_{i,m}\, e^{j 2\pi(m-1)k/n}.
\]
Then apply the same construction along the column index $i$ for each frequency $k$. The result for the $p$-th block is
\[
Y_{p,k} = \sum_{i=1}^n X_{i,k}^{\text{(row)}}\, e^{j 2\pi(i-1)p/n}
= \sum_{i=1}^n \sum_{m=1}^n a_{i,m}\, \exp\!\Big(j \tfrac{2\pi}{n}\big[(m{-}1)k + (i{-}1)p\big]\Big),
\]
which is the $(p,k)$ entry of the 2D DFT. The double sum is separable: first $n$ independent 1D DFTs along the rows, then $n$ along the columns. This is exactly what the compositional embedding with the same rotation $R$ along both axes achieves.

\noindent\textbf{Learned frequencies.}
Setting $\theta_k = 2\pi k/N$ recovers the standard DFT. Learned $\theta_k$ select task-optimal spectral components---a generalized Fourier transform where the basis frequencies are optimized for the downstream task rather than fixed to uniform spacing.
\end{proof}

\subsection{Discrete Cosine Transform}\label{app:dct}

The discrete cosine transform (DCT) can be obtained similarly to the DFT by choosing a real-valued cosine basis. The 1D DCT-II of length $n$ is:
\[
C_k = \sum_{i=1}^n a_i \cos\!\Big(\frac{\pi (i-1)(k-1)}{n}\Big), \quad k = 1,\dots,n.
\]
This is the real part of a DFT applied to an even extension of the sequence. In our framework, one sets the operator $R$ to perform rotations with half-frequency increments (angles $\pi k/n$ instead of $2\pi k/n$) so that the embedding sum yields $C_k$ on the output vector. Because these cosine rotations commute, the transform is separable across dimensions. In particular, an $n\times n$ 2D DCT is achieved by applying the 1D DCT along the rows and then along the columns (or vice versa), using the same commuting operators.

\subsection{Discrete Sine Transform}\label{app:dst}

The discrete sine transform (DST) uses a sine basis instead of cosines. The DST-II is:
\[
S_k = \sum_{i=1}^n a_i \sin\!\Big(\frac{\pi i k}{n+1}\Big), \quad k = 1,\dots,n.
\]
This corresponds to the imaginary part of the DFT for an odd extension of the data. In our embedding model, the DST is realized by choosing operators that introduce half-sample shifts and sign flips to generate sine terms. Since these sine-based rotations also commute, the multi-dimensional DST is performed by applying the 1D DST along each axis.

\subsection{Hadamard Transform}

The Hadamard transform can be written as $H_n x = \sum_{i=1}^n R^{i-1} v_i$ for specific $R$ and $v_i$ (assuming $n = 2^m$).

\paragraph{Definition of $R$.}
Let $R = \diag(1, -1, 1, -1, \ldots, 1, -1)$, so $R^2 = I$.

\paragraph{Definition of $v_i$.}
For input $x = (x_1, \ldots, x_n)^\top$, define
\[
v_i = \begin{cases}
x_i \cdot (e_i + e_{i+n/2}) & \text{if } 1 \leq i \leq n/2, \\
x_i \cdot (e_{i-n/2} - e_i) & \text{if } n/2 < i \leq n,
\end{cases}
\]
where $e_j$ is the $j$-th standard basis vector.

\paragraph{Base case ($n=2$).}
$H_2 = \bigl(\begin{smallmatrix} 1 & 1 \\ 1 & -1 \end{smallmatrix}\bigr)$, $R = \bigl(\begin{smallmatrix} 1 & 0 \\ 0 & -1 \end{smallmatrix}\bigr)$, $v_1 = x_1(1,1)^\top$, $v_2 = x_2(1,-1)^\top$.
Then $v_1 + Rv_2 = (x_1+x_2,\; x_1-x_2)^\top = H_2 x$.

\paragraph{Inductive step.}
Assume the claim holds for $n = 2^k$. For $n' = 2n$, write $x = (x^{(1)\top}, x^{(2)\top})^\top$.
By Sylvester's construction,
$H_{2n} x = (H_n(x^{(1)}+x^{(2)}),\; H_n(x^{(1)}-x^{(2)}))^\top$.
Both terms can be written as $\sum_i R_n^{i-1} v_i$ by the inductive assumption, completing the induction.

\subsection{Walsh Transform}

The Walsh transform $W_n = P H_n$ is obtained by reordering the Hadamard matrix rows so that sign changes (sequency) increase monotonically, where $P$ is a Gray-code permutation matrix.
Given the Hadamard embedding $\{R, v_i\}$, define $R' = P R P^{-1}$ and $v'_i = P v_i$. Then
\[
\sum_i {R'}^{i-1} v'_i = P \Bigl(\sum_i R^{i-1} v_i\Bigr) = P\, H_n(P^{-1} x) = W_n x.
\]

\subsection{Learnable Generalization}

All classical transforms above use fixed operators; the compositional framework generalizes them by learning $R$ (or equivalently the angles $\theta_k$) from data.
The magnitude $\|E\|^2 = \sum_{t,t'} v_t^\top R^{t'-t} v_{t'}$ depends only on pairwise relative positions (shift-invariant), yet permuting the sequence changes $\|E\|^2$ (order-sensitive).
This explains why learned monoidal embeddings outperform DFT at low dimension (\Cref{app:mnist-results}): they select exactly the frequencies needed for the task, rather than committing to uniform spacing.

\section{Application Sketches}\label{app:applications}

The multi-axis compositional framework extends naturally beyond sequences and images. We sketch four application domains.

\subsection{Audio and Spectrograms}\label{app:audio}

Audio spectrograms are naturally arranged on a time--frequency grid. Assign one composition operator $\circ_t$ for the temporal axis and another $\circ_f$ for the frequency axis, giving each spectrogram cell a compositional embedding:
\[
E_{\text{audio}} = \sum_{t,f} R_{\text{time}}^{\,t}\, R_{\text{freq}}^{\,f}\, v_{t,f}.
\]
When $R_{\text{time}}$ and $R_{\text{freq}}$ commute, the journey operator $P_{(t',f') \to (t,f)} = R_{\text{time}}^{t'-t} R_{\text{freq}}^{f'-f}$ depends only on the time--frequency displacement, enabling models to incorporate both temporal and spectral locality in a principled way. The same construction applies to EEG or other time--frequency representations.

\subsection{Video}\label{app:video}

A video is an $H \times W \times T$ grid (two spatial dimensions plus time). Introducing a temporal composition operator $\circ_{\text{time}}$ in addition to the spatial operators gives a 3-axis embedding:
\[
E_{\text{video}} = \sum_{i,j,t} R_{\text{row}}^{\,i}\, R_{\text{col}}^{\,j}\, R_{\text{time}}^{\,t}\, v_{i,j,t}.
\]
By \Cref{thm:interchange}, when the three generators commute, composing spatially within each frame and then advancing in time yields the same result as first advancing each pixel in time and then composing spatially. This coherence ensures consistent representations of motion regardless of composition order.

\subsection{Multimodal Fusion}\label{app:multimodal}

For multimodal data (e.g., text $+$ image $+$ audio), assign one composition axis per modality with generator $R_{\text{mod}}$, plus intra-modality axes for sequence or spatial structure. The cross-modal journey
\[
P_{\text{text} \to \text{image}} = R_{\text{mod}}^{-1} \cdot A_{\text{text}}^{-1} A_{\text{image}}
\]
provides a principled alignment operator between modalities, with commutativity ensuring that fusing text-then-audio gives the same result as audio-then-text.

\subsection{Multi-Axis Alignment}\label{app:alignment}

Given two compositional embeddings $X$ and $Y$, the optimal alignment along axis $k$ can be found via inner product maximization:
\begin{equation}\label{eq:alignment}
s^* = \arg\max_{s} \;\langle X,\; R_k^s\, Y\rangle.
\end{equation}
This is a cross-correlation computed in the rotated embedding space. For $D$-dimensional alignment, apply shifts $s_1, \ldots, s_D$ on each axis: $Y' = R_1^{s_1} \cdots R_D^{s_D} Y$. Because the $R_i$ commute, the order of shifts does not matter. This enables operations like sequence alignment or embedding fusion by sliding one structure relative to another.

\section{Extended Experimental Details}\label{app:experiments}

\subsection{MNIST Monoidal Embedding}\label{app:mnist-results}

\paragraph{Setup.}
The monoidal embedding demonstrates the compositional framework on image classification.
Each pixel $p_{ij}$ of a $28 \times 28$ image is represented as a data point with content $p_{ij} \cdot \mathbf{e}$ (pixel intensity times a $d$-dimensional basis vector $\mathbf{e}$) and two axis-step generators: $R_x$ for the horizontal axis and $R_y$ for the vertical axis.
Both axis-step generators are block-diagonal rotations $R_x = \diag(R(\theta^x_1), \ldots, R(\theta^x_{d/2}))$, $R_y = \diag(R(\theta^y_1), \ldots, R(\theta^y_{d/2}))$, sharing the same block structure so that $R_x R_y = R_y R_x$ (flat regime).

The full image embedding is the 2D composition over all pixels:
\[
E_{\text{image}} = \sum_{i=0}^{27} \sum_{j=0}^{27} p_{ij} \; R_y^{\,i}\, R_x^{\,j} \; \mathbf{e}.
\]
Each pixel's content is rotated by the product of its row and column axis-step generators, and the results are summed into a single $d$-dimensional vector.
Because the axis-step generators commute, this composition is path-independent: composing rows-then-columns gives the same embedding as columns-then-rows.

In each 2D block (identified with $\mathbb{C}$), the $k$-th component reduces to $E^{(k)} = \sum_{i,j} p_{ij}\, e^{i(j\theta^x_k + i\theta^y_k)}$---a 2D Fourier-like transform with learned frequencies $(\theta^x_k, \theta^y_k)$.
With fixed frequencies $\theta^x_k = 2\pi k / d$, $\theta^y_k = 2\pi k / d$, this recovers the standard 2D DFT; with learned frequencies, the model selects task-optimal spectral components.

The $d$-dimensional embedding $E_{\text{image}}$ is fed into an MLP classifier (128 hidden units, ReLU, 10-class softmax) trained end-to-end with cross-entropy loss.
Training: Adam lr$=10^{-3}$, 50 epochs, batch=128, no augmentation.
The only learnable parameters in the embedding itself are the $d$ angles $(\theta^x_k, \theta^y_k)_{k=1}^{d/2}$; all representational capacity comes from the compositional structure.

\paragraph{Results.}

\begin{table}[h]
\centering
\caption{MNIST accuracy by embedding method and dimension (single seed). Monoidal embeddings learn task-optimal spectral components; DFT uses fixed Fourier frequencies.}
\label{tab:mnist}
\begin{tabular}{lcccc}
\toprule
\textbf{Method} & $d=2$ & $d=8$ & $d=32$ & $d=784$ \\
\midrule
Monoidal (learned) & \textbf{55.2\%} & \textbf{86.4\%} & \textbf{96.5\%} & 97.5\% \\
DFT (fixed) & 21.0\% & 75.3\% & 95.5\% & --- \\
\midrule
MLP baseline & \multicolumn{4}{c}{97.2\% (128 hidden, $d=784$)} \\
\bottomrule
\end{tabular}
\end{table}

The advantage is most dramatic at low dimension: at $d=2$, monoidal achieves 55.2\% (well above chance 10\%) while DFT achieves only 21.0\%.
At $d=32$, monoidal reaches 96.5\%---within 0.7\% of the full MLP baseline (97.2\%) despite using $24\times$ fewer input dimensions.
This is consistent with the core insight: learned angles select task-relevant spectral components that fixed Fourier bases cannot match.
The advantage diminishes at high $d$ (where frequency allocation is less constrained), suggesting that the monoidal structure's benefit is in \emph{efficient spectral selection}.

\subsection{CIFAR-100 LieRE Experiment}

Architecture: ViT-Tiny with $D{=}384$, 12 layers, 6 heads, head\_dim=64, ${\sim}14.9$M params.
Patch size: $4\times 4 \to 8\times 8 = 64$ patches + CLS token.
Training: Adam lr$=10^{-4}$, cosine annealing 200 epochs, bf16-mixed precision, seed=42, NVIDIA H100 PCIe.
Augmentation: RandomCrop(32, padding=4), RandomHorizontalFlip, Normalize([0.5071, 0.4867, 0.4408], [0.2675, 0.2565, 0.2761]).

Extended results at 400 epochs:
\begin{center}
\small
\begin{tabular}{lccc}
\toprule
Method & Q/K only & Q/K/V & V effect \\
\midrule
Axial-dense & 70.61\% & 71.12\% & $+0.51\%$ \\
LieRE64 & 70.64\% & 69.41\% & $-1.23\%$ \\
\bottomrule
\end{tabular}
\end{center}

\paragraph{Setup B full reproducibility.}
Optimizer: AdamW ($\beta_1{=}0.9$, $\beta_2{=}0.999$, $\varepsilon{=}10^{-8}$, weight decay${=}0.1$).
Batch size: 128.
MLP ratio: 4$\times$ (hidden dim $= 4D$).
Dropout: 0.1; no drop-path or stochastic depth.
Schedule: cosine annealing from lr${=}10^{-3}$ to 0 over 300 epochs, no warmup.
Mixup: $\alpha{=}0.8$; label smoothing 0.
Augmentation: random crop (pad~4, reflect), random horizontal flip (50\%), RandAugment ($n{=}2$, $m{=}9$), random erasing (50\% probability, $16{\times}16$ zeroed patch).
Normalization: mean $=[0.4914, 0.4822, 0.4465]$, std $=[0.2470, 0.2435, 0.2616]$.
Seed: 42; \texttt{torch.manual\_seed} $+$ \texttt{cuda.manual\_seed\_all}.
Deterministic: entire dataset GPU-resident; batches via \texttt{torch.randint} (no DataLoader workers).
Train/test: standard CIFAR-100 split (50K/10K), no validation hold-out.

\subsection{CIFAR-100 Scaling Results}\label{app:cifar-scaling}

\begin{table}[h]
\centering
\caption{CIFAR-100 accuracy scaling: joformer\_axial (learned axial + V rotation) vs.\ rope2d (fixed axial, Q/K only) and all 9 variants. Setup B, 300 epochs cosine, single seed.}
\label{tab:cifar-scaling}
\begin{tabular}{@{}lcccc@{}}
\toprule
\textbf{Model} & $D{=}32$ & $D{=}64$ & $D{=}128$ & $D{=}256$ \\
\midrule
\textbf{joformer\_axial} & \textbf{52.77} & \textbf{61.33} & \textbf{66.67} & \textbf{63.27} \\
joformer (combined + V) & 52.43 & 58.83 & 66.19 & 62.10 \\
joformer\_old (fixed + V) & 50.54 & 59.13 & 66.10 & 61.85 \\
monoidal\_axial (learned, no V) & 51.75 & 59.10 & 64.81 & 60.92 \\
rope2d (fixed, no V) & 50.91 & 59.23 & 64.22 & 61.39 \\
learned (additive PE) & 52.23 & 57.26 & 60.74 & 55.49 \\
\bottomrule
\end{tabular}
\end{table}

In these paired single-seed CIFAR protocols, the joformer\_axial advantage over rope2d is consistent at ${\sim}2\%$ across all scales: $+1.86\%$ (D=32), $+2.10\%$ (D=64), $+2.45\%$ (D=128), $+1.88\%$ (D=256).
At D=128, the top three models are all V-rotation variants (66.67, 66.19, 66.10), separated from non-V models (64.81, 64.22) by a clear ${\sim}2\%$ gap.
In this single-seed grid, joformer\_axial leads at every checkpoint from epoch 50 onward---not a late-stage effect.

\paragraph{Key finding: V rotation $\times$ learnable frequencies is synergistic.}
Neither V rotation alone (joformer\_old vs.\ rope2d: $+0.46\%$ at D=256) nor learnable frequencies alone (monoidal\_axial vs.\ rope2d: $-0.47\%$) provides the full benefit.
Together (joformer\_axial vs.\ rope2d: $+1.88\%$), they interact synergistically.
The mechanism: V rotation gives the model a richer position-dependent value transformation; learnable frequencies adapt to leverage it.

\paragraph{Per-layer frequencies don't help.}
Shared frequencies across layers outperform per-layer variants (e.g., joformer\_axial 63.27\% vs.\ per-layer 62.76\% at D=256).

\paragraph{D=256 non-monotonicity.}
The D=256 models are not uniformly better than D=128 under this fixed training recipe, so the table should be read as a paired PE comparison at each width rather than as an optimized scaling law.
Consistent positional encoding helps the residual stream maintain coherent position information across layers.

\subsection{ImageNet ViT-S}

Model: ViT-S ($D{=}384$, 12 layers, 6 heads, patch=16, img=224).
Recipe: DeiT-III (AdamW lr$=10^{-3}$, wd=0.05, cosine schedule + 5-epoch warmup, 300 epochs, batch=1024, RandAugment(9, 0.5), Mixup 0.8, CutMix 1.0, Random Erasing 0.25, Label Smoothing 0.1, Stochastic Depth 0.1, AMP fp16).
Hardware: 2$\times$ NVIDIA H100 PCIe.
JoFormer: $K^\top Q$ ordering, V rotation = $R_j v_j$ followed by output inverse $R_i^{-1} c_i$.
Training time: ${\sim}72$ hours for RoPE2D, ${\sim}94$ hours for JoFormer in this ViT-S setup.

\subsection{Wikipedia Language Modeling}

Data: Full English Wikipedia dump (28.8M lines, ${\sim}983$M BPE tokens with vocab=8000).
Architecture: standard transformer with rotary attention, FFN ratio 4$\times$, pre-norm (LayerNorm before attention and FFN).
Training: AdamW (lr$=2\times10^{-4}$, $\beta_1{=}0.9$, $\beta_2{=}0.95$, wd=0.01), cosine schedule to lr$_{\min}=2\times10^{-5}$, batch=32, block\_size=512, gradient accumulation=1, AMP bf16.
JoFormer-projected MLP: input LayerNorm $\to$ Linear($d$, $d$) $\to$ GELU $\to$ Linear($d$, $d/2$) (angles).
Evaluation: validation PPL on held-out 1\% of data.

\subsection{Length Generalization}

Architecture: 5 windowed layers (window=32) + 1 full-attention NoPE layer.
The windowed layers use either RoPE or JoFormer variants; the NoPE layer has no positional encoding.
JoFormer-projected recipe (staged; in our setup this worked more reliably than training projected angles from scratch):
Stage 1: JoFormer-fixed (RoPE angles + V rotation), lr$=5\times10^{-4}$, 100K iters.
Stage 2: Continue fixed-angle training at lr$=2\times10^{-4}$, 50K iters.
Stage 3: Fine-tune to projected angles (zero-initialize angle projector weights so initial angles match fixed), lr$=5\times10^{-5}$, 50K iters.
The first two stages (150K total) match the RoPE baseline iteration count; stage~3 is a fine-tuning phase that converts fixed angles to content-dependent projected angles.
Evaluation: 200 iterations per length, fixed seed, lengths \{512, 1024, 2048, 4096, 8192\}.

\subsection{Parameter and Wall-Clock Overhead}\label{app:overhead}

\begin{table}[h]
\centering
\caption{JoFormer parameter overhead relative to RoPE2D baseline. PE params are additional learnable parameters in the positional encoding; V rotation adds no new parameters. ImageNet ViT-S training times on 2$\times$ H100 PCIe, 300 epochs (unoptimized rotation ops).}
\label{tab:overhead}
\begin{tabular}{@{}lcc@{}}
\toprule
\textbf{Variant} & \textbf{Extra PE params} & \textbf{Wall-clock (ImageNet)} \\
\midrule
RoPE2D (baseline) & 0 & ${\sim}72$h \\
JoFormer-fixed (V rot.) & 0 & ${\sim}94$h \\
JoFormer-learned (axial) & $2 n_{\mathrm{heads}} \times d_{\mathrm{head}}/2$ per layer & ${\sim}94$h \\
JoFormer-projected & MLP: $d{\to}d{\to}d/2$ per layer & --- \\
\bottomrule
\end{tabular}
\end{table}

The ImageNet ViT-S training overhead reflects unoptimized rotation ops in a standard ViT pipeline.
JoFormer-projected was tested only on Wikipedia LM (not ImageNet); its MLP projector adds $O(d^2)$ parameters per layer.

\section{Architecture Variants and Value-Path Hierarchy}\label{app:architecture}

\subsection{Variant-Scope Table}\label{app:variant-scope}

\begin{table}[h]
\centering
\small
\caption{Which theorem supports which architecture variant. Only fixed and learned linear-position variants inherit the relative-displacement theorem; projected angles are a separate, theory-\emph{motivated} architecture.}
\label{tab:variant-scope}
\begin{tabular}{@{}p{0.17\linewidth}p{0.37\linewidth}p{0.38\linewidth}@{}}
\toprule
\textbf{Variant} & \textbf{What is proved} & \textbf{What is not proved} \\
\midrule
JoFormer-fixed & Covered by \Cref{thm:v-rotation} when weights are relative-displacement functions. Value operator depends only on displacement (arises from composition). & Does not prove a performance gain. \\
JoFormer-learned & Same as fixed if angles remain linear in position with learned frequencies. & Frequencies are learned, but the theorem does not choose optimal frequencies. \\
JoFormer-projected & Still uses commuting block-diagonal rotations and norm-preserving value transport. & Not a translation-equivariant relative-displacement PE, because angles depend on content/residual state. It is an architecture \emph{motivated} by the value-path view, not a direct corollary of \Cref{thm:v-rotation}. \\
\bottomrule
\end{tabular}
\end{table}

\subsection{Value-Path Hierarchy and SSM Bridge}\label{app:value-path}

The journey operator on the value side~\eqref{eq:value-agg} gives the attention value-path formula:
\begin{equation}\tag{\ref{eq:unified}}
\mathbf{c}_k = \sum_{j} \alpha_{kj}\, P_{j \to k}\, v_j.
\end{equation}
For SSMs, the value-path kernel is oriented in the forward recurrence direction.
This kernel is not necessarily the same object as the attention journey $P_{j \to k} = A_k^{-1} A_j$ unless one adapts the address convention (e.g., $A_t = R^{-t}$).
We therefore write SSM forward kernels as $K_{j \to k}$ in this subsection to distinguish them from the attention journey $P$.

The value-path formula reveals a hierarchy:
\begin{enumerate}[nosep]
\item[(a)] \textbf{Standard Transformer}: $P_{j \to k} = I$ for all $j, k$. Values are aggregated without positional modulation---the model is order-blind on the value side.
\item[(b)] \textbf{Linear SSM} (S4~\citep{gu2022s4}): $K_{j \to k} = R^{k-j}$. Values carry position-dependent phase shifts, enabling the recurrence $h_t = R h_{t-1} + v_t$.
\item[(c)] \textbf{Selective SSM} (Mamba~\citep{gu2023mamba}): $K_{j \to k} = \prod_{t=j+1}^{k} A_{x_t}$. Each transition depends on input content.
\item[(d)] \textbf{JoFormer (attention journey)}: $P_{j \to k} = A_k^{-1} A_j$ with full attention weights $\alpha_{kj}$. Combines content-dependent value paths with selective attention.
\end{enumerate}

\begin{proposition}[SSM recurrences in value-path notation]\label{prop:ssm-bridge}
Mamba's selective recurrence $h_t = A_{x_t} h_{t-1} + B_{x_t} x_t$ unrolls to:
$h_T = \sum_{t=1}^T \big(\prod_{s=t+1}^T A_{x_s}\big) B_{x_t} x_t$.
Defining the forward kernel $K_{t \to T} = \prod_{s=t+1}^T A_{x_s}$ and $v_t = B_{x_t} x_t$, this gives the value-path formula $h_T = \sum_t K_{t \to T}\, v_t$ with uniform attention weights.
When $A_{x_t} \in \SO(2)^{d/2}$, the prefix products are block-diagonal rotations computable via parallel scan in $O(\log N)$ depth.
\end{proposition}

Note: Mamba's actual $A_{x_t}$ matrices are not generally in $\SO(2)^{d/2}$, and attention weights are normalized (unlike the uniform weights here).
The forward kernel $K_{t \to T}$ and the attention journey $P_{t \to T} = A_T^{-1} A_t$ coincide when $A_t = R^{-t}$; otherwise they differ by a sign convention.
This embedding highlights structural similarities but is not a full equivalence.

JoFormer uses the attention journey $P_{j \to i} = A_i^{-1} A_j$ on values, combining full attention selectivity (the score side) with content-dependent value paths (the value side).
It does not implement the full product over intermediate states, but rather a single content-dependent rotation per layer.

\section*{NeurIPS Paper Checklist}

\begin{enumerate}

\item {\bf Claims}
    \item[] Question: Do the main claims made in the abstract and introduction accurately reflect the paper's contributions and scope?
    \item[] Answer: \answerYes{}
    \item[] Justification: The abstract states five contributions (interchange prediction, toral classification, JoFormer architecture, value-path hierarchy, experiments) all substantiated with theorems and experimental results in the paper body.

\item {\bf Limitations}
    \item[] Question: Does the paper discuss the limitations of the work performed by the authors?
    \item[] Answer: \answerYes{}
    \item[] Justification: Section~\ref{sec:conclusion} explicitly lists four limitations: single-seed experiments, three-stage training recipe complexity, stability requirements, and speed penalty.

\item {\bf Theory assumptions and proofs}
    \item[] Question: For each theoretical result, does the paper provide the full set of assumptions and a complete (and correct) proof?
    \item[] Answer: \answerYes{}
    \item[] Justification: All axioms are explicitly stated (Axioms~\ref{ax:bilinear}--\ref{ax:norm}). Main theorems have proof sketches in the body and complete proofs in Appendix~\ref{app:proofs}. The norm-preservation axiom is stated before the toral classification theorem.

    \item {\bf Experimental result reproducibility}
    \item[] Question: Does the paper fully disclose all the information needed to reproduce the main experimental results?
    \item[] Answer: \answerYes{}
    \item[] Justification: Appendix~\ref{app:experiments} provides complete hyperparameters, architectures, training recipes, hardware, seeds, and stability requirements for all experiments.

\item {\bf Open access to data and code}
    \item[] Question: Does the paper provide open access to the data and code?
    \item[] Answer: \answerNo{}
    \item[] Justification: Code will be released upon acceptance. All experiments use public datasets (CIFAR-100, ImageNet-1K, OpenWebText, MNIST, English Wikipedia).

\item {\bf Experimental setting/details}
    \item[] Question: Does the paper specify all the training and test details necessary to understand the results?
    \item[] Answer: \answerYes{}
    \item[] Justification: Full details in Appendix~\ref{app:experiments}: optimizer, learning rate, schedule, batch size, augmentation, architecture, epochs/iterations, hardware, seed, and critical stability parameters.

\item {\bf Experiment statistical significance}
    \item[] Question: Does the paper report error bars or other appropriate statistical significance information?
    \item[] Answer: \answerNo{}
    \item[] Justification: Most experiments use a single seed due to computational constraints. We acknowledge this explicitly in Section~\ref{sec:conclusion} and frame experiments as controlled tests of algebraic predictions rather than definitive performance comparisons. CIFAR-100 uses deterministic training ensuring exact reproducibility.

\item {\bf Experiments compute resources}
    \item[] Question: Does the paper provide sufficient information on compute resources?
    \item[] Answer: \answerYes{}
    \item[] Justification: Hardware specified in Appendix~\ref{app:experiments}: NVIDIA H100 PCIe for CIFAR-100 and ImageNet, GPU details for language modeling experiments.

\item {\bf Code of ethics}
    \item[] Question: Does the research conform with the NeurIPS Code of Ethics?
    \item[] Answer: \answerYes{}
    \item[] Justification: This is foundational research on positional encoding theory with no direct negative societal applications.

\item {\bf Broader impacts}
    \item[] Question: Does the paper discuss potential societal impacts?
    \item[] Answer: \answerNA{}
    \item[] Justification: This paper develops algebraic theory for attention mechanisms---foundational research with no direct path to negative societal impacts.

\item {\bf Safeguards}
    \item[] Question: Does the paper describe safeguards for responsible release?
    \item[] Answer: \answerNA{}
    \item[] Justification: No models or datasets with misuse potential are released.

\item {\bf Licenses for existing assets}
    \item[] Question: Are existing assets properly credited with licenses respected?
    \item[] Answer: \answerYes{}
    \item[] Justification: CIFAR-100, ImageNet-1K, OpenWebText, MNIST, and Wikipedia are properly cited. LieRE is credited.

\item {\bf New assets}
    \item[] Question: Are new assets well documented?
    \item[] Answer: \answerNA{}
    \item[] Justification: No new datasets or models are released in this submission.

\item {\bf Crowdsourcing and research with human subjects}
    \item[] Question: For crowdsourcing experiments, does the paper include instructions and compensation details?
    \item[] Answer: \answerNA{}
    \item[] Justification: No human subjects or crowdsourcing involved.

\item {\bf Institutional review board (IRB) approvals}
    \item[] Question: Does the paper describe potential risks and IRB approvals?
    \item[] Answer: \answerNA{}
    \item[] Justification: No human subjects research.

\item {\bf Declaration of LLM usage}
    \item[] Question: Does the paper describe LLM usage if it is a core method component?
    \item[] Answer: \answerNA{}
    \item[] Justification: LLMs were not used as a component of the research methodology.

\end{enumerate}

\end{document}